\documentclass[runningheads]{llncs}
\usepackage[T1]{fontenc}
\usepackage{graphicx}
\usepackage{amssymb}
\usepackage{amsmath}

\usepackage{tikz}
\usetikzlibrary{automata, positioning, arrows, fit, cd}

\usepackage{comment}

\usepackage{pgfplots}
\pgfplotsset{compat=1.18}
\usepgfplotslibrary{fillbetween}

\usepackage{subcaption}
\usepackage{wrapfig}
\usepackage{thm-restate}
\usepackage{booktabs}
\usepackage{multirow}
\usepackage{makecell}

\newcommand{\act}{\mathit{Act}}
\newcommand{\initdist}{\mu}

\makeatletter
\newsavebox{\@brx}
\newcommand{\llangle}[1][]{\savebox{\@brx}{\(\m@th{#1\langle}\)}%
  \mathopen{\copy\@brx\mkern2mu\kern-0.9\wd\@brx\usebox{\@brx}}}
\newcommand{\rrangle}[1][]{\savebox{\@brx}{\(\m@th{#1\rangle}\)}%
  \mathclose{\copy\@brx\mkern2mu\kern-0.9\wd\@brx\usebox{\@brx}}}
\makeatother

\newcommand{\drnote}[1]{
}

\usepackage{color}
\usepackage[hypertexnames=false]{hyperref}

\usepackage{cleveref}
\usepackage{orcidlink} 
\begin{document}
%
\title{Value Functions as Supermartingale Certificates}
\titlerunning{Value Functions as Supermartingale Certificates}
%
\author{
Alessandro Abate\inst{1}\orcidlink{0000-0002-5627-9093}
\and
Daniel Contro\inst{2}\orcidlink{0009-0004-8222-2797}
\and
Mirco Giacobbe\inst{2}\orcidlink{0000-0001-8180-0904}
\and
Agustín Martínez-Suñé\inst{1}\orcidlink{0000-0003-1806-6932}
\and
Diptarko Roy\inst{2}\orcidlink{0009-0003-4306-2076}
}
\authorrunning{A. Abate et al.}
%
\institute{
University of Oxford, UK
\and
University of Birmingham, UK
}
\maketitle              
\begin{abstract}

Certification methods for stochastic systems provide sufficient proof rules, based on real-valued supermartingale certificates, to determine the almost-sure satisfaction of $\omega$-regular properties (and therefore of linear temporal logic) over general state spaces, encompassing both countably infinite and continuous state spaces. 
Conversely, reinforcement learning (RL) methods for $\omega$-regular tasks have received considerable attention, but they typically lack formal guarantees that the learned policy satisfies the specification, except possibly for finite state and action spaces.
We bridge these two lines of research by establishing a novel theoretical connection: under an appropriate reward, the value function associated to a policy that almost surely satisfies an $\omega$-regular property encodes a Streett supermartingale certificate for that specification.
%
Our results, validated experimentally on finite Markov decision processes, hold for finite, countably infinite, and continuous state spaces, suggesting a principled route to certificate synthesis via RL.
%
\end{abstract}
%
%
%

\section{Introduction}

Proof certificates have been widely studied as a means of establishing that a system satisfies a specification with mathematical guarantees. More precisely, in the context of infinite-state systems with stochastic behaviour, the verification problem reduces to computing supermartingale certificates.
Proof rules for supermartingale certificates have been developed for verifying persistence properties, recurrence properties \cite{chakarovDeductiveProofsAlmost2016}, and, more recently, $\omega$-regular properties \cite{abateStochasticOmegaRegularVerification2024,DBLP:conf/cav/AbateGR25,henzingerSupermartingaleCertificatesQuantitative2025}, which include those expressible in linear temporal logic~(LTL).
At its core, the problem of supermartingale certificate synthesis amounts to finding a function $W : S \to \mathbb{R}$ that maps the system's state space to the real numbers and satisfies a set of proof rules determined by the system dynamics and the property under verification.
We focus on Streett supermartingales~\cite{abateStochasticOmegaRegularVerification2024},
which certify $\omega$-regular properties---subsuming LTL and, in particular, reachability and safety.
A key challenge is that existing synthesis methods 
rely on constrained optimisation, which does not scale well to large or continuous state spaces.

Reinforcement learning (RL) is a standard machine learning paradigm for solving sequential decision-making problems in environments with possibly stochastic dynamics. The environment is modeled as a Markov decision process (MDP), where the goal is to select actions that maximize the expected discounted return under a given reward function. A key component of RL is the estimation of the value function $V : S \to \mathbb{R}$ of a given policy, which maps states to the expected discounted return obtained by following such policy.
An increasing body of work proposes the use of model-free RL to synthesize controllers 
for tasks specified in LTL~\cite{hahnFaithfulEffectiveReward2020,hasanbeigLCRLCertifiedPolicy2022,sadighLearningBasedApproach2014}.  
These approaches typically synchronize the MDP state on the fly with an $\omega$-regular automaton that corresponds to the LTL specification, and assign rewards based on the automaton and its acceptance conditions.
Crucially, while these methods might provide guarantees that the optimal policy maximises the 
probability of satisfaction, they lack formal guarantees that a trained policy---which 
may not be optimal---satisfies the specification with any given probability, except possibly in 
finite state and action spaces.

This gap motivates our work: we study the connection between Streett supermartingales and value functions in the discounted return setting.

\begin{description}
    \item[\textbf{Theory}] We propose two reward designs and prove that, under either one, the value function of a policy that almost surely satisfies an LTL specification encodes a valid Streett supermartingale (Theorems~\ref{thm:v-streett-supermartingale-bscc} and~\ref{thm:v-streett-sm-spec-reward}). The first requires knowledge of the absorbing sets induced by the policy; the second requires only knowledge of the specification. Both results hold for finite, countably infinite, and continuous state spaces.
    
    \item[\textbf{Experiments}] We validate both reward designs on a stochastic grid-world MDP across $\omega$-regular specifications spanning the full Manna-Pnueli hierarchy~\cite{mannaHierarchyTemporalProperties1990}. 
    We compute our value functions by solving the Bellman system, and verify our certificates by checking the supermartingale proof rules; we confirm satisfaction probabilities independently via the PRISM model checker~\cite{kwiatkowskaPRISM40Verification2011}.
    The code to reproduce our experiments is publicly available.
    
\end{description}
Together, these results open a principled route to supermartingale certificate synthesis via RL, reducing the problem from constrained optimisation to policy evaluation followed by a certificate check, with the potential to scale to large and continuous state spaces via data-driven methods.

\section{$\omega$-Regular Verification with Supermartingales}
\label{sec:background}

\begin{definition}[Markov Decision Process]
Let $S$ be a state space and $\act$ an action space, each either 
finite or countable (equipped with the discrete $\sigma$-algebra) 
or a Borel subset of a Euclidean space (equipped with the Borel $\sigma$-algebra).
A (discrete-time) \emph{Markov decision process (MDP)} is a tuple
$\mathcal{M} = (S, \act, P, \initdist)$
where:
\begin{itemize}
    \item $P$ is a stochastic kernel on $S$ given $S \times \act$, i.e.,
    for each $(s,a) \in S \times \act$, the map $P(\cdot \mid s,a)$ is a 
    probability measure on $S$, and for each measurable $E \subseteq S$, 
    the map $(s,a) \mapsto P(E \mid s,a)$ is measurable;
    \item $\initdist$ is an initial probability distribution on $S$. 
\end{itemize}
When additionally equipped with a measurable reward function 
$r : S \times \act \times S \to \mathbb{R}$, we refer to the tuple 
$(S, \act, P, r, \initdist)$ as a \emph{rewardful MDP}.
\end{definition}


A (stationary) \emph{stochastic policy} is a stochastic kernel
$\pi(\cdot \mid s)$ on $\act$ given $S$.  
A \emph{deterministic policy} is a measurable mapping $\pi : S \to \act$, 
identified with the kernel $\pi(\cdot \mid s) = \delta_{\pi(s)}$, 
the Dirac measure at $\pi(s)$.
%
Given a policy $\pi$, the controlled process $\{S_t, A_t\}_{t\ge 0}$
satisfies
\(
\Pr(A_t \in D \mid S_t = s) = \pi(D \mid s),
\)
\(
\Pr(S_{t+1} \in E \mid S_t = s, A_t = a) = P(E \mid s,a),
\)
for all measurable $D \subseteq \act$ and $E \subseteq S$.
Together with $\initdist$, which governs the distribution of $S_0$, 
this induces a probability measure $\Pr_\pi$ on the run space
$\Omega = S^\omega$, $\mathcal{F} = \mathcal{B}(S)^{\otimes \omega}$,
where runs $\rho = (s_0, s_1, s_2, \ldots)$ evolve according to $P$ under $\pi$.



Linear Temporal Logic (LTL) provides a specification language for
predicating over infinite executions.  Let $\mathrm{AP}$ be a finite set of
atomic propositions.  The syntax of LTL formulas is given by
\[
\varphi ::=\;
    p \mid \neg\varphi \mid \varphi_1 \land \varphi_2
    \mid \mathbf{X}\,\varphi \mid \varphi_1 \mathbin{\mathbf{U}} \varphi_2,
\qquad p \in \mathrm{AP}.
\]
where $\mathbf{X}$ is the ``next'' operator and $\mathbf{U}$ is the
``until'' operator.  Standard derived operators such as
$\mathbin{\mathbf{F}}\varphi$ (eventually) and $\mathbin{\mathbf{G}}\varphi$ (always) are
defined as usual.  For a detailed account of the syntax and semantics
of LTL, we refer the reader to the standard literature~\cite{baierPrinciplesModelChecking2008}.

A \emph{labelled MDP} is an MDP $(S, \act, P, \initdist)$ additionally equipped 
with a finite set of atomic propositions $\mathit{AP}$ and a labelling function 
$L : S \to 2^{\mathit{AP}}$, which assigns to each state the set of atomic propositions that hold 
in such state. A run $\rho = (s_0, s_1, s_2, \ldots) \in S^\omega$ then induces a 
corresponding word $L(\rho) = L(s_0)\,L(s_1)\,L(s_2)\cdots$ over $2^{\mathit{AP}}$.

Any LTL formula $\varphi$ can be translated into a deterministic $\omega$-regular automaton
$\mathcal{A}_\varphi$ whose accepted words are precisely those satisfying
$\varphi$~\cite{baierPrinciplesModelChecking2008}, reducing verification to the analysis of the
product between the labelled MDP and $\mathcal{A}_\varphi$. Here we use the Streett acceptance condition,
which is expressive enough to capture all $\omega$-regular (and, consequently, all LTL) properties. 

\begin{definition}[Deterministic Streett Automaton]
A \emph{deterministic Streett automaton (DSA)} is a tuple
$\mathcal{A} = (Q, 2^{\mathit{AP}}, \delta, q_0, \mathit{Acc})$
where:
\begin{itemize}
    \item $Q$ is a finite set of states;
    \item $2^{\mathit{AP}}$ is the input alphabet;
    \item $\delta : Q \times 2^{\mathit{AP}} \to Q$ is a deterministic transition function;
    \item $q_0 \in Q$ is the initial state;
    \item $\mathit{Acc} = \{(A_1, B_1), \ldots, (A_k, B_k)\}$ is a finite set of 
    \emph{Streett pairs}, with $A_i, B_i \subseteq Q$ for each $i$.
\end{itemize}
A run $\varrho = (q_0, q_1, q_2, \ldots) \in Q^\omega$ of $\mathcal{A}$ on a word
$\sigma = \sigma_0\sigma_1\sigma_2\cdots \in (2^{\mathit{AP}})^\omega$ satisfies
$q_{t+1} = \delta(q_t, \sigma_t)$ for all $t \geq 0$.
The run is \emph{accepting} if for all pairs $(A_i, B_i) \in \mathit{Acc}$,
\[
\inf(\varrho) \cap A_i \neq \emptyset \implies \inf(\varrho) \cap B_i \neq \emptyset,
\]
where $\inf(\varrho) \subseteq Q$ denotes the set of states visited infinitely often.
\end{definition}

Streett supermartingale certificates are real-valued functions that decrease in expectation within designated regions of the state space and which serve as sufficient conditions for almost-sure satisfaction of $\omega$-regular properties~\cite{abateStochasticOmegaRegularVerification2024,kuraHierarchySupermartingalesoRegular2025}.
\begin{definition}[Streett supermartingales~{\cite{abateStochasticOmegaRegularVerification2024,kuraHierarchySupermartingalesoRegular2025}}]
\label{def:streett-supermartingale}
Let $\mathcal{M} = (S, \act, P, \initdist)$ be an MDP,
let $\pi$ be a stationary policy,
and let $(A,B)$ be a Streett pair with $A,B \subseteq S$.
Suppose there exist a measurable set $I \subseteq S$, a nonnegative
measurable function $W : I \to \mathbb{R}_{\ge 0}$, and constant
$\varepsilon$ such that:
\begin{align}
&\initdist(I) = 1, 
\label{def:streett-supermartingale:eq-invariant-init}
\\[4pt]
&\forall s \in I:\;
   \Pr\nolimits_\pi(S_1 \in I \mid S_0 = s) = 1 
\label{def:streett-supermartingale:eq-invariance}
\\[4pt]
&\forall s \in (A \setminus B) \cap I:\;
    \mathbb{E}_\pi[W(S_1) \mid S_0 = s] \le W(s) - \varepsilon,
\label{def:streett-supermartingale:eq-strictly-decreasing}
\\[4pt]
&\forall s \in I \setminus (A \cup B):\;
    \mathbb{E}_\pi[W(S_1) \mid S_0 = s] \le W(s).
\label{def:streett-supermartingale:eq-non-increasing}
\end{align}
Then $W$ is a \emph{Streett supermartingale} for $(A,B)$.  In particular,
$W$ certifies that the process controlled by $\pi$ satisfies the Streett objective almost surely:
\[
\Pr_\pi\!\left(
  \sum_{t=0}^\infty \mathbf{1}_A(S_t) = \infty
  \implies
  \sum_{t=0}^\infty \mathbf{1}_B(S_t) = \infty
\right) = 1.
\]
We refer to the set $I$
as a \emph{supporting invariant}.
\end{definition}

To verify that a policy $\pi$ on an MDP $\mathcal{M}$ almost surely satisfies 
an LTL specification $\varphi$, one applies this definition to the product MDP 
$\mathcal{M} \otimes \mathcal{A}_\varphi$, replacing $S$ with $S \times Q$ and 
each Streett pair $(A_i, B_i)$ with its lifted counterpart 
$(S \times A_i, S \times B_i) \in \mathit{Acc}^\otimes$.
Since the DSA acceptance condition requires all $k$ Streett pairs to be satisfied 
simultaneously, it suffices to find a Streett supermartingale for each pair 
$(S \times A_i, S \times B_i) \in \mathit{Acc}^\otimes$ independently.

\section{Value Functions as Supermartingale Certificates}
\label{sec:value-function-certificates}
Our main contribution establishes that searching for a Streett supermartingale 
certificate reduces to standard RL with an appropriate choice of reward function.
We propose two reward designs under which the value function of a policy that 
almost surely satisfies a given specification encodes a valid Streett supermartingale 
certificate. The first reward design requires knowledge of the absorbing sets induced by the policy 
under analysis, while the second requires only knowledge of the specification, with the
tradeoffs between the two different designs discussed in \Cref{sec:conclusions}.
We begin by recalling some central definitions and introducing the fundamental lemmas on which the proofs of the theorems rely.

\begin{definition}[Value Function]\label{def:val-func}
Given a rewardful MDP $(S, \act, P, r, \initdist)$,
a policy $\pi$, and discount factor $\gamma \in [0,1)$, the
state-value function is
\[
V^\pi(s)
=
\mathbb{E}_\pi\!\left[
    \sum_{t=0}^\infty \gamma^t\, r(S_t, A_t, S_{t+1})
    \,\middle|\, S_0 = s
\right].
\]
It satisfies the Bellman equation
\(
V^\pi(s)
= \mathbb{E}_\pi\!\big[r(s, A_0, S_1) + \gamma\, V^\pi(S_1) \mid S_0 = s\big].
\)
When the discount factor is state-dependent, i.e.\ $\Gamma \colon S \to [0,1]$,
the Bellman equation becomes
\(
V^\pi(s)
= \mathbb{E}_\pi\!\big[r(s, A_0, S_1) + \Gamma(s)\, V^\pi(S_1) \mid S_0 = s\big].
\)
\end{definition}
%

Our first reward design requires knowledge
of the absorbing sets under the transition kernel induced by a policy.

\begin{definition}[Absorbing set]
\label{def:absorbing-set}
Let $\mathcal{M} = (S, \act, P, r, \initdist)$ be an MDP,
let $\pi$ be a policy. A measurable set $U \subseteq S$ is said to
be absorbing if and only if
\[
\forall s \in U, \; \Pr_\pi(S_1 \in U \mid S_0 = s) = 1 
\]
\end{definition}

\begin{definition}[Largest absorbing subset]
\label{def:largest-absorbing-subset}
Let $\mathcal{M} = (S, \act, P, r, \initdist)$ be an MDP,
let $\pi$ be a policy, and let $X \subseteq S$ be measurable.
The \emph{largest absorbing subset} of $X$ under~$\pi$ is%
%
%
\[
\mathrm{Abs}_\pi(X) := \bigl\{\, s \in S : \Pr_\pi \bigl( \forall t \ge 0 : S_t \in X \mid S_0 = s \bigr) = 1 \,\bigr\}.
\]
\end{definition}

Intuitively, the largest absorbing subset of $X \subseteq S$ is
the set of states from which the MDP remains in $X$ for all time, almost surely. In other words, in terms of the function $h^\pi : S \to [0, 1]$ given by $h^\pi(s) := \Pr_\pi( \forall t \geq 0 \colon S_t \in X \mid S_0 = s)$, which evaluates the probability of remaining forever in $X$ from state $s$, we see that $\mathrm{Abs}_\pi(X)$ is the set of states at which $h^\pi$ evaluates to 1.
We observe that $\mathrm{Abs}_\pi(X) \subseteq X$, which follows from the fact that the trajectory cannot remain in $X$ for all $t \geq 0$ if $S_0 \notin X$. To prove that $\mathrm{Abs}_\pi(X)$ is indeed absorbing (\Cref{def:absorbing-set}), we recall that $h^\pi$ satisfies the martingale property within $X$ \cite[Theorem 4.1.3]{douc2018markov}, namely:
\(
    \forall s \in X:\;
    \mathbb{E}_\pi[h^\pi(S_1) \mid S_0 = s] = h^\pi(s).
\)
For $s \in \mathrm{Abs}_\pi(X)$ we have $h^\pi(s) = 1$, so $\mathbb{E}_\pi[h^\pi(S_1) \mid S_0 = s] = 1$. Since $h^\pi \le 1$, this forces $h^\pi(S_1) = 1$ almost surely, i.e.\ $S_1 \in \mathrm{Abs}_\pi(X)$ almost surely. 

Central to our approach is reasoning about the first time a policy enters a set where positive reward is observed.
\begin{definition}[First hitting time {\cite[Def.~3.2.2]{douc2018markov}}]
\label{def:first-hitting-time}
Let $\mathcal{M} = (S, \act, P, r, \initdist)$ be an MDP, let $\pi$ be a policy and
$U \subseteq S$ be a measurable set.
We say that
\[
\tau_U := \inf\{t \geq 0 : S_t \in U\}
\]
is the first hitting time of $U$.
\end{definition}

The following lemmata underpin our first reward design; complete proofs are provided in 
Appendix~\ref{appendix:proofs}.
We first establish that, when reward is collected only upon visiting a 
measurable set~$U$, the value function can be expressed in terms of the 
first hitting time of~$U$.
\begin{restatable}{lemma}{lemmaDecomposeV}
\label{lem:decompose-v}
Let $U \subseteq S$ be measurable, and let $r(s, a, s') = \mathbf{1}_{\{s \in U\}}$. For any policy $\pi$ and state \(s \in S\),%
\[
V^\pi(s)
= \mathbb{E}_\pi\!\left[
   \gamma^{\tau_U} 
   V^\pi\!\big(S_{\tau_U}\big)
   \mathbf{1}_{\{\tau_U<\infty\}}
   \mid S_0 = s
  \right].
\]
\end{restatable}

A key aspect of the Streett supermartingale conditions lies in the expected change after one step.
Building on the decomposition in \Cref{lem:decompose-v}, we can express this expected one-step
change in the value function in terms of the first visit to~\(U\). 

\begin{restatable}{lemma}{lemmaDeltaFormula}
\label{lem:delta-formula}
Let $U \subseteq S$ be measurable, and let $r(s, a, s') = \mathbf{1}_{\{s \in U\}}$.
The one-step change $\Delta \colon S \setminus U \to \mathbb{R}$, defined as
$\Delta(s) := \mathbb{E}_\pi\!\big[V^\pi(S_1)\mid S_0=s\big] - V^\pi(s)$, satisfies
\[
\Delta(s)
=
\mathbb{E}_\pi\!\Big[
   (1-\gamma)\,
   \gamma^{\tau_U(S_1)}\,
   V^\pi(S_{\tau_U(S_1)})\,
   \mathbf{1}_{\{\tau_U(S_1)<\infty\}}
   \;\Big|\; S_0=s
\Big],
\]
where $\tau_U(S_1)$ denotes the first hitting time of $U$ starting at $S_1$.%
\footnote{More precisely, $\tau_U(S_1) = \tau_U \circ \theta_1$, 
where $\theta_1 : S^\omega \to S^\omega$ is the shift operator defined by 
$\theta_1(s_0, s_1, s_2, \ldots) = (s_1, s_2, \ldots)$ \cite[Definition 3.1.8]{douc2018markov}, and $\tau_U$ (the 
first hitting time of $U$) is applied to the shifted process.}
\end{restatable}

We now analyse the sign and magnitude of the expected one-step change in the value function under specific conditions relevant to our reward design.
%
%
Within an absorbing set $I$ from which the rewarded set $U$ is almost surely reachable, the expected one-step change
is strictly positive until $U$ is reached.

\begin{restatable}{lemma}{lemmaVStrictlyIncreasing}
\label{lem:v-strictly-increasing}
Let $r(s, a, s') = \mathbf{1}_{\{s \in U\}}$ with $U \subseteq S$ measurable. Let $I \subseteq S$ 
be an absorbing set under a policy $\pi$. Assume that
$\forall s \in I \setminus U : \Pr_\pi \left(\tau_U < \infty \mid S_0 = s\right) = 1.$
Then, the expected one-step change is positive:
\[
\forall s \in I \setminus U : \mathbb{E}_\pi\!\big[V^\pi(S_1)\mid S_0=s\big] - V^\pi(s) > 0.
\]
\end{restatable}

When the rewarded set $U$ is absorbing and its reward is constant, the value function is also constant within $U$.

\begin{restatable}{lemma}{lemmaVNonIncreasing}
\label{lem:v-non-increasing}
Let $U \subseteq S$ be measurable, assume that
$s \in U \implies r(s, a, s') = 1$
and
$U$ is absorbing under a policy $\pi$.
Then the expected one-step change is zero:
\[
\forall s \in U : \mathbb{E}_\pi\!\big[V^\pi(S_1)\mid S_0=s\big] - V^\pi(s) = 0.
\]
\end{restatable}

Finally, when the expected hitting time of the rewarded set~$U$ is uniformly bounded from above over a set~$F$, the expected one-step change of the value function admits a uniform lower bound~$\varepsilon > 0$.

\begin{restatable}{lemma}{lemmaVStrictlyIncreasingUniformBound}
\label{lem:v-strictly-increasing-uniform-bound}
Let $r(s, a, s') = \mathbf{1}_{\{s \in U\}}$ with $U \subseteq S$.
Let  $F \subseteq S$ be a set,
and assume there exists a constant $\bar H < \infty$ such that
$\forall s \in 
F:
\mathbb{E}_\pi\!\big[\,\tau_U \mid S_0 = s\,\big] \le \bar H.$
Then, the expected one-step change admits a uniform positive lower bound:
\[
\forall s \in 
F
\setminus U:
\mathbb{E}_\pi\!\big[V^\pi(S_1)\mid S_0=s\big] - V^\pi(s)
\ge
\varepsilon,
\]
where $\varepsilon := (1-\gamma)\,\gamma^{\bar H - 1} > 0$.
\end{restatable}

Building on \Cref{lem:v-strictly-increasing,lem:v-non-increasing,%
lem:v-strictly-increasing-uniform-bound}, we now present our first reward design, which leverages knowledge
of the absorbing sets induced by a policy $\pi$, to certify almost-sure Streett acceptance.


\begin{theorem}
\label{thm:v-streett-supermartingale-bscc}
Let 
\(\mathcal{M} = (S, \act, P, r, \initdist)\)
be an MDP, and let \(\pi\) be any policy.
Let \(I \subseteq S\) be a set and let \((A,B)\) be a Streett pair with 
\(A,B \subseteq S\).
Assume:

\begin{enumerate}

    \item \textbf{Invariant:}
    \label{thm:v-streett-supermartingale-bscc-infty:hyp-invariant}
    \[
    I \textit{ is absorbing and }
    \initdist(I)=1,
    \]

    \item \textbf{Almost-sure reachability:}
    \label{thm:v-streett-supermartingale-bscc-infty:hyp-reachability-1}
    \[
    \forall s\in I,\ 
    \Pr_\pi\!\big(
        \tau_{\,B \cup \mathrm{Abs}_\pi((A\cup B)^c)}
        < \infty   
        \mid S_0 = s
    \big)
    = 1.
    \]
    \item \textbf{Uniform bound on hitting time:}
    \label{thm:v-streett-supermartingale-bscc-infty:hyp-reachability-1-strict}
    There exists a constant \(\bar H < \infty\) such that for all 
    \(s \in (A \setminus B)\cap I\),
    \[
    \mathbb{E}_\pi\!\big[
        \tau_{B \cup \mathrm{Abs}_\pi((A\cup B)^c)}
        \mid S_0 = s
    \big]
    \;\le\;
    \bar H .
    \]

    \item \textbf{Reward:}
    \label{thm:v-streett-supermartingale-bscc-infty:hyp-reward}
    \[
    r(s,a,s') \;=\; 
    \mathbf{1}_{\{\,s \in B \cup \mathrm{Abs}_\pi((A\cup B)^c)\,\}}.
    \]

\end{enumerate}
Then \(W(s):= C - V^{\pi}(s)\), with \(C := \tfrac{1}{1-\gamma}\),
is a Streett supermartingale for \((A,B)\).
\end{theorem}

The key insight behind the reward design is to assign positive reward whenever 
the process reaches $B$ or enters a state from which neither $A$ nor $B$ will 
ever be visited again---formally, the largest absorbing subset 
$\mathrm{Abs}_\pi((A \cup B)^c)$.
Note that 
hypothesis~\ref{thm:v-streett-supermartingale-bscc-infty:hyp-reachability-1} 
\textbf{(\textit{almost-sure reachability})} is equivalent to almost-sure Streett 
satisfaction within the supporting invariant $I$, while 
hypothesis~\ref{thm:v-streett-supermartingale-bscc-infty:hyp-reachability-1-strict} 
\textbf{(\textit{uniform bound on hitting time})} is a stronger requirement needed 
to establish the strictly decreasing condition~\eqref{def:streett-supermartingale:eq-strictly-decreasing} of the Streett supermartingale.

\begin{proof}[\Cref{thm:v-streett-supermartingale-bscc}]
We need to show that $W(s) = C - V^\pi(s)$ is a Streett supermartingale (\Cref{def:streett-supermartingale}) for the Streett pair $(A, B)$.

Conditions~\eqref{def:streett-supermartingale:eq-invariant-init} and \eqref{def:streett-supermartingale:eq-invariance} are identical to hypothesis~\ref{thm:v-streett-supermartingale-bscc-infty:hyp-invariant}, so we have to verify that \(W(s) := C - V^\pi(s)\) is nonnegative and satisfies conditions~\eqref{def:streett-supermartingale:eq-strictly-decreasing}%
~and~\eqref{def:streett-supermartingale:eq-non-increasing}.

\paragraph{Nonnegativity.} The reward definition (hypothesis~\ref{thm:v-streett-supermartingale-bscc-infty:hyp-reward}) induces (by \Cref{def:val-func}) the value function:
\begin{align}
V^\pi(s)
= 
\mathbb{E}_\pi\!\left[
  \sum_{t=0}^\infty 
  \gamma^t\, 
  \mathbf{1}_{\{S_t \in {B \cup \mathrm{Abs}_\pi((A\cup B)^c)}\}}  \,\middle|\, S_0 = s
\right].
\end{align}
Since the term $\mathbf{1}_{\{S_t \in {B \cup \mathrm{Abs}_\pi((A\cup B)^c)}\}} \leq 1$ for all \(t\),
it follows that \(V^\pi(s) \le \sum_{t=0}^\infty \gamma^t = C := \frac{1}{1-\gamma}\),
so \(W(s) = C - V^\pi(s) \ge 0\) for all \(s \in I\).

\paragraph{Condition~\eqref{def:streett-supermartingale:eq-strictly-decreasing}.}
The reward function is equal to $1$ within $B \cup \mathrm{Abs}_\pi((A\cup B)^c)$
(hypothesis~\ref{thm:v-streett-supermartingale-bscc-infty:hyp-reward} \textbf{(\textit{reward})})
and the expected hitting time of $B \cup \mathrm{Abs}_\pi((A\cup B)^c)$ is uniformly bounded from $(A \setminus B) \cap I$
(hypothesis~\ref{thm:v-streett-supermartingale-bscc-infty:hyp-reachability-1-strict} \textbf{(\textit{uniform bound on hitting time})}). Hence, we can use 
\Cref{lem:v-strictly-increasing-uniform-bound} with 
$U = B \cup \mathrm{Abs}_\pi((A\cup B)^c)$
and
$F = (A \setminus B)$ to obtain
a constant $\varepsilon$
such that
\begin{align}
\forall s \in \bigl((A \setminus B) \cap I\bigr) \setminus U
\mathpunct.
\mathbb{E}_\pi[V^\pi(S_1)\mid S_0=s] - V^\pi(s)
\;\ge\; \varepsilon.
\end{align}
States in $A \setminus B$ are by definition not in $B$, and 
$\mathrm{Abs}_\pi((A\cup B)^c) \subseteq (A \cup B)^c$ contains no states in $A$,
therefore $((A \setminus B) \cap I)\cap U = \emptyset$.
Hence $((A \setminus B) \cap I) \setminus U = (A \setminus B) \cap I$ and the quantifier can be simplified to
\begin{align}
\forall s \in (A \setminus B) \cap I
\mathpunct.
\mathbb{E}_\pi[V^\pi(S_1)\mid S_0=s] - V^\pi(s)
\;\ge\; \varepsilon.
\end{align}
Substituting $W = C - V^\pi$ gives
\begin{align}
\forall s \in (A \setminus B) \cap I
\mathpunct.
\mathbb{E}_\pi[W(S_1)\mid S_0=s]
\;\le\;
W(s) - \varepsilon,
\end{align}
which establishes condition~\eqref{def:streett-supermartingale:eq-strictly-decreasing}.

\paragraph{Condition \eqref{def:streett-supermartingale:eq-non-increasing}.}
The reward function is equal to $1$ within $B \cup \mathrm{Abs}_\pi((A\cup B)^c)$
(hypothesis~\ref{thm:v-streett-supermartingale-bscc-infty:hyp-reward} \textbf{(\textit{reward})}) and
$B \cup \mathrm{Abs}_\pi((A\cup B)^c$ is almost-surely reached from any state within the supporting invariant $I$
(hypothesis~\ref{thm:v-streett-supermartingale-bscc-infty:hyp-reachability-1} \textbf{(\textit{almost-sure reachability})}).
Therefore, thanks to \Cref{lem:v-strictly-increasing} with $U=B \cup \mathrm{Abs}_\pi((A\cup B)^c)$ we obtain
\begin{align}
\forall s \in I \setminus [{B \cup \mathrm{Abs}_\pi((A\cup B)^c)}]
\mathpunct.
\mathbb{E}_\pi[V^\pi(S_1)\mid S_0=s] - V^\pi(s) > 0.
\end{align}
Substituting $W = C - V^\pi$ gives
\begin{align}\label{eqn:notin-Abs-AuBc}
\forall s \in I \setminus [
B \cup 
\mathrm{Abs}_\pi( (A \cup B)^c )
] 
\mathpunct.
\mathbb{E}_\pi[W(S_1)\mid S_0=s] -W(s) < 0.
\end{align}
By hypothesis \ref{thm:v-streett-supermartingale-bscc-infty:hyp-invariant} \textbf{(\textit{invariant})} the set $I$ is absorbing,
moreover the set $\mathrm{Abs}_\pi( (A \cup B)^c )$ is also absorbing by construction (\Cref{def:largest-absorbing-subset}).
Therefore, $I \cap \mathrm{Abs}_\pi( (A \cup B)^c )$ is absorbing under the transition kernel.
We note that $r(s, a, s^\prime) = 1$ for every $s \in I \cap \mathrm{Abs}_\pi( (A \cup B)^c), 
    s^\prime \in S$ and $a \in \act$ by our definition of the reward function in hypothesis~\ref{thm:v-streett-supermartingale-bscc-infty:hyp-reward} \textbf{(\textit{reward})}.
    We can then apply \Cref{lem:v-non-increasing} with $U=I \cap \mathrm{Abs}_\pi( (A \cup B)^c )$
    to conclude that
\begin{align}
    \forall s \in I \cap \mathrm{Abs}_\pi( (A \cup B)^c)
    \mathpunct. 
    \mathbb{E}\!\big[V^\pi(S_1)\mid S_0=s\big] - V^\pi(s) = 0.
\end{align}
Substituting $W = C - V^\pi$ gives
\begin{align}\label{eqn:result-lem4}
    \forall s \in I \cap \mathrm{Abs}_\pi( (A \cup B)^c)
    \mathpunct. 
    \mathbb{E}\!\big[W(S_1)\mid S_0=s\big] - W(s)= 0.
\end{align}
Finally, to verify condition~\eqref{def:streett-supermartingale:eq-non-increasing}, we consider an arbitrary $s \in I \setminus  (A \cup B)$:
\begin{itemize}
    \item \textbf{Case} $s \in  \mathrm{Abs}_\pi( (A \cup B)^c )$. Since $s \in I \cap \mathrm{Abs}_\pi( (A \cup B)^c )$, condition \eqref{def:streett-supermartingale:eq-non-increasing} follows from \cref{eqn:result-lem4}.
    
    \item \textbf{Case} $s \notin 
    \mathrm{Abs}_\pi( (A \cup B)^c )
    $. %
    Since $s \in I \setminus  (A \cup B)$ and $s \notin 
    \mathrm{Abs}_\pi( (A \cup B)^c )
    $,
    condition \eqref{def:streett-supermartingale:eq-non-increasing} follows from \cref{eqn:notin-Abs-AuBc}.%
\qed
\end{itemize}

\end{proof}

Since any $\omega$-regular (and LTL) specification can be encoded in a DSA $\mathcal{A}$,
proof certificates for such properties can be obtained by applying
Theorem~\ref{thm:v-streett-supermartingale-bscc} to the product MDP
$\mathcal{M} \otimes \mathcal{A}$, where the Streett pairs are directly inherited  from the automaton.

\begin{corollary}
\label{cor:v-streett-supermartingale-product}
Let $\mathcal{M} \otimes \mathcal{A}_\varphi$ be a product MDP with a DSA,
with state space $S \times Q$ and lifted Streett pairs $\mathit{Acc}^\otimes$.
Let $\pi$ be a stationary policy and let
$(S \times A, S \times B) \in \mathit{Acc}^\otimes$ be a lifted Streett pair.
Let $I \subseteq S \times Q$ be a set. Assume that the hypotheses of
Theorem~\ref{thm:v-streett-supermartingale-bscc} hold for the product MDP
$\mathcal{M} \otimes \mathcal{A}_\varphi$ in place of $\mathcal{M}$, the lifted pair
$(S \times A, S \times B)$ in place of $(A, B)$, and the reward
\[
r^\otimes\bigl((s,q), a, (s',q')\bigr)
=
\mathbf{1}_{\{\,(s,q) \in (S \times B) \cup \mathrm{Abs}_\pi((S \times (A \cup B))^c)\,\}}.
\]
Then $W(s, q) := C - V^\pi(s, q)$, with $C := \tfrac{1}{1-\gamma}$,
is a Streett supermartingale for $(S \times A, S \times B)$ on
$\mathcal{M} \otimes \mathcal{A}_\varphi$.
\end{corollary}

\begin{example}
\label{ex:bscc-reward}
The following example, presented in Abate~et~al.~\cite[Example 3]{abateStochasticOmegaRegularVerification2024},
considers a simple Markov process over one real-valued variable $x$ and control parameter $\kappa$,
described by the following stochastic difference equation:
\begin{equation}
  x_{t+1} = \kappa \cdot x_t + w_t,\qquad
  x_0 = 100,\qquad
  w_t \sim \mathrm{Uniform}(-0.1,0.1).
\end{equation}

Consider the stabilize-while-avoid property
$\Phi = (x \geq -1)\,\mathsf{U}\,\mathsf{G}(-1 \leq x \leq 1)$,
that requires the system to avoid $x < -1$ until it stabilizes within $(-1 \leq x \leq 1)$, which corresponds to the DSA in \Cref{fig:dsa-stabilize-avoid}, with a single Streett pair $(\{q_0, q_2\}, \emptyset)$.
The policy that assigns $\kappa = 0.5$ satisfies $\Phi$ almost surely and the relevant absorbing set established by \Cref{thm:v-streett-supermartingale-bscc} is
\(
\mathrm{Abs}(S\times\{q_1\}) = \{-1 \leq x < 1\}\times q_1
\).
Hence, the reward function  becomes 
\begin{equation*}
r\bigl((x,q), a, (x',q')\bigr)
=
\begin{cases}
1 & \text{if } -1 \le x < 1 \text{ and } q = q_1,\\
0 & \text{otherwise.}
\end{cases}
\end{equation*}

Consider the supporting invariant $I(x, q_0) = (
-0.2 \leq x \leq 100 )$,
$I(x, q_1) = (-0.2 \leq x \leq 0.9)$, and
$I(x, q_2) = \mathit{false}$.
Within this invariant, the hitting time of $\mathrm{Abs}(S \times \{q_1\})$ is at most $7$ for every initial state: 
from $x = 100$, exactly $\lceil \log_2(100) \rceil = 7$ steps are needed before $0.5^7 \cdot 100 + 0.2 < 1$, where $0.2$ bounds the worst-case accumulated stochastic noise. 
This uniform finite bound establishes both hypothesis~\ref{thm:v-streett-supermartingale-bscc-infty:hyp-reachability-1} 
(almost-sure reachability) and hypothesis~\ref{thm:v-streett-supermartingale-bscc-infty:hyp-reachability-1-strict} 
(uniform bound on hitting time).
The theorem yields the Streett supermartingale $W(s, q) := C - V^\pi(s, q)$,
with $C := \tfrac{1}{1-\gamma}$.
Inside the absorbing set ($q = q_1$), the process collects reward $1$ for an infinite number of steps,
achieving the maximum discounted return $V^\pi = C$, so $W = 0$ and remains $0$,
satisfying the non-increasing condition \eqref{def:streett-supermartingale:eq-non-increasing} in \Cref{def:streett-supermartingale}.
Outside the absorbing set ($q = q_0$), the reward is $0$ but the contraction
$\kappa = 0.5$ brings the process closer to the absorbing set at each step.
Since the same eventual return is discounted by one fewer factor of $\gamma$, $V^\pi$ strictly increases and $W$ strictly decreases.
The worst-case strict decrease occurs at $x = 100$, where the expected hitting time is $7$ and the one-step gain in $V^\pi$ is $\varepsilon = \gamma^6$,
satisfying the strictly decreasing condition \eqref{def:streett-supermartingale:eq-strictly-decreasing} in \Cref{def:streett-supermartingale}. \qed
%

\begin{figure}[t]
\centering
\begin{tikzpicture}[
    >=stealth,
    node distance=3.5cm,
    every state/.style={
    minimum size=0.8cm},
    every edge/.style={draw,
    ->},
    label/.style={font=\small},
    cert/.style={font=\small, text=blue!60!black}
]

\node[state, initial, initial text={}] (q0) {$q_0$};
\node[state, right of=q0] (q1) {$q_1$};
\node[state, right=2cm of q1] (q2) {$q_2$};

\path[->] (q0) edge[loop above] node[above] {$x \geq 1$} (q0);
\path[->] (q1) edge[loop above] node[above] {$-1 \leq x < 1$} (q1);
\path[->] (q2) edge[loop above] node[above] {$\mathrm{true}$} (q2);

\path[->] (q0) edge[bend left=15] node[above] {$-1 \leq x < 1$} (q1);
\path[->] (q1) edge[bend left=15] node[below] {$x \geq 1$} (q0);

\path[->] (q1) edge node[above] {$x < -1$} (q2);

\path[->] (q0) edge[bend right=35] node[below] {$x < -1$} (q2);

\node[cert, below=0.3cm] at (q0.south) {
};
\node[cert, below=0.3cm] at (q1.south) {
};
\node[cert, below=0.3cm] at (q2.south) {
};

\node[right=2em of q2, font=\normalsize] {$\mathrm{Acc} = \{(\{q_0, q_2\}, \emptyset)\}$};

\end{tikzpicture}
\caption{Deterministic Streett Automaton for $(x \geq -1) \mathsf{U} \mathsf{G}(-1 \leq x \leq 1)$.} 
\label{fig:dsa-stabilize-avoid}
\end{figure}
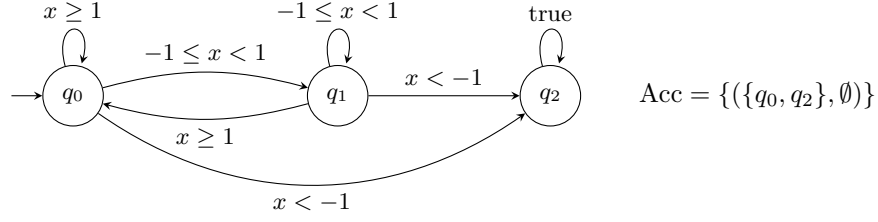    
\end{example}

As a special case of \Cref{thm:v-streett-supermartingale-bscc}, when the 
Streett pair takes the form $(S, T)$---so that the acceptance condition 
requires $T$ to be visited infinitely often---the theorem simplifies to 
the following.
Crucially, in this case the reward function requires no knowledge of the absorbing
sets under the policy, depending only on the target set $T$.
\begin{corollary}
\label{cor:v-recurrence-supermartingale}
Let $\mathcal{M} = (S, \act, P, r, \initdist)$ be an MDP,
let $\pi$ be a policy, let $I \subseteq S$, and let $T \subseteq S$ be a
measurable target set. Assume:
\begin{enumerate}
    \item \textbf{Invariant:}
    $I$  is absorbing and 
    $\initdist(I)=1$.
    \item \textbf{Almost-sure recurrence:}
    From every $s \in I$.
    \(
    \Pr_\pi\!\bigl(\tau_T(s) < \infty\bigr) = 1.
    \)
    \item \textbf{Uniform bound on hitting time:}
    There exists a constant $\bar{H} < \infty$ such that for all
    $s \in I \setminus T$,
    \(
    \mathbb{E}_\pi\!\bigl[\tau_T(s)\bigr] \le \bar{H}.
    \)
    \item \textbf{Reward:}
    \(
    r(s, a, s') = \mathbf{1}_{\{s \in T\}}.
    \)
\end{enumerate}
Then $W(s) := C - V^\pi(s)$, with $C := \tfrac{1}{1-\gamma}$, is a supermartingale for the Streett pair $(S,T)$, certifying the recurrence of $T$.
\end{corollary}


We now turn to our second reward design, which requires only knowledge 
of the specification and does not rely on absorbing sets. However, it 
imposes an alternative assumption on the satisfying policy, expressed 
in terms of the following random variable.
\begin{definition}[{\cite[Def.~3.1]{kuraHierarchySupermartingalesoRegular2025}}]
For measurable sets $A, B \subseteq S$, we define a random variable $\mathrm{step}^{(A,B)} : S^\omega \to \mathbb{N}$ as the number of $A$-steps until reaching $B$. For a run $\rho = (s_0, s_1, s_2, \ldots) \in S^\omega$,
\[
\mathrm{step}^{(A,B)}(\rho) := \left|\left\{i \mid 0 \leq i < \min\{j \mid s_j \in B \land j \geq 0\} \land s_i \in A\right\}\right|
\]
\end{definition}

The following random variable, representing the number of visits to a 
set $U$ before time $t$, is needed in the proof of our second reward design.
\begin{definition}
For a measurable set $U \subseteq S$ and $t \in \mathbb{N}$, the random variable $N_t^{U} : S^\omega \to \mathbb{N}$ is defined as
\[
N_t^{U}(\rho)
:=
\left|
\left\{
i \;\middle|\;
0 \le i < t
\;\land\;
s_i \in U
\right\}
\right|,
\]
where $\rho = (s_0, s_1, s_2, \ldots) \in S^\omega$.
\end{definition}

The following auxiliary lemma specifies our second reward design and establishes a uniform lower bound on the resulting value function.

\begin{restatable}{lemma}{lemmaSpecRewardLowerBound}
\label{lem:v-streett-sm-spec-reward:lower-bound}
Let $\mathcal{M} = (S, \act, P, r, \initdist)$ be an MDP, \(\pi\) a policy, 
\(I \subseteq S\) an absorbing set with $\mu(I) = 1$, and \((A,B)\) a Streett pair
with $A,B \subseteq S$. If there exists a constant $\bar{H} < \infty$ such that
\(
\forall s \in I:
\mathbb{E}_{\pi}\left[\mathrm{step}^{(A,B)} \mid S_0 = s\right] \leq \bar{H},
\)
the state-dependent discount factor is
\[
\Gamma(s)=
\begin{cases}
\gamma, & r(s,a,s')\neq 0,\\[2pt]
1, & \text{otherwise},
\end{cases}
\qquad 0<\gamma<1,
\]
and the reward function is
\[
r(s,a,s') = \mathit{r_B}\,\mathbf{1}_{\{s\in B\}}
- \mathit{r_{A\setminus B}}\,\mathbf{1}_{\{s\in A\setminus B\}},
\]
where $\mathit{r_B}, \mathit{r_{A\setminus B}} > 0$ and
\(
\frac{r_B}{r_{A\setminus B}} \geq 
\frac{1}{1-\gamma}
\left(
\frac{1}{\gamma^{\bar{H}+1}}
- 1
\right),
\)
then 
%
for all $s \in I$,
\[
V^{\pi}(s) \;\ge\; -\, r_{A\setminus B}\,\frac{1 - \gamma^{\bar{H}}}{1 - \gamma}.
\]
\end{restatable}

The lower bound follows from two cases: when $A$ is visited infinitely often, the ratio condition on $r_B / r_{A\setminus B}$ ensures that the positive reward from visiting $B$ compensates the accumulated negative reward from $A\setminus B$ between consecutive visits to $B$; when $A$ is visited only finitely often, the bound $\bar{H}$ on the expected number of visits to $A$ before the next visit to $B$ controls the accumulated penalty in the tail, where no further visits to $B$ occur.


We now state our second main result, which, unlike the first, does not rely on absorbing sets.

\begin{theorem}
\label{thm:v-streett-sm-spec-reward}
Let 
\(\mathcal{M} = (S, \mathcal{A}, P, r, \initdist)\)
be an MDP,
let \(\pi\) be a policy,
\(I \subseteq S\) be a set,
and \((A,B)\) a Streett pair
with $A,B \subseteq S$.
Assuming:
\begin{enumerate}
    \item \textbf{Invariant:}
    \label{thm:v-streett-sm-spec-reward:hyp-invariant}
        \[
    I \textit{ is absorbing and }
    \initdist(I)=1,
    \]
    \item \textbf{Positive recurrence bound:}
    \label{thm:v-streett-sm-spec-reward:hyp-recurrence}
    There exists a constant $\bar{H} < \infty$ such that
    \[
    \forall s \in I:
    \mathbb{E}_{\pi}\left[\mathrm{step}^{(A,B)} \mid S_0 = s\right] \leq \bar{H},
    \]
    \item \textbf{State-dependent discount factor:}
    \label{thm:v-streett-sm-spec-reward:hyp-discount}
    \[
    \Gamma(s)=
    \begin{cases}
    \gamma, & r(s,a,s')\neq 0,\\[2pt]
    1, & \text{otherwise},
    \end{cases}
    \qquad 0<\gamma<1,
    \]
    \item \textbf{Reward:}
    \label{thm:v-streett-sm-spec-reward:hyp-reward}
    \[
    r(s,a,s') = \mathit{r_B}\,\mathbf{1}_{\{s\in B\}}
    - \mathit{r_{A\setminus B}}\,\mathbf{1}_{\{s\in A\setminus B\}},
    \]
    with $\mathit{r_B}, \mathit{r_{A\setminus B}} > 0$ and
    \(
    \frac{r_B}{r_{A\setminus B}} \geq 
    \frac{1}{1-\gamma}
    \left(
    \frac{1}{\gamma^{\bar{H}+1}}
    - 1
    \right).
    \)
\end{enumerate}
Then \(W(s) := C - V^{\pi}(s)\), with \(C := \tfrac{\mathit{r_B}}{1-\gamma}\), is a Streett supermartingale for \((A,B)\).
\end{theorem}

The state-dependent discount factor leaves transitions outside $A \cup B$ undiscounted, satisfying the supermartingale non-increasing condition~\eqref{def:streett-supermartingale:eq-non-increasing} with equality. On $A \setminus B$, since the policy almost surely eventually leaves this set,
$V^\pi$ increases in expectation: each step contributes one fewer negative
reward and one fewer discounting factor ($\Gamma = \gamma$) to the sum.
However, establishing a uniform $\varepsilon > 0$ for condition~\eqref{def:streett-supermartingale:eq-strictly-decreasing} requires controlling how negative $V^\pi$ can become. 
Hypothesis~\ref{thm:v-streett-sm-spec-reward:hyp-recurrence} (\textit{positive recurrence bound}) plays this role: $\bar{H}$ limits the accumulated negative reward, while the ratio condition on $r_B / r_{A\setminus B}$ ensures sufficient compensation,
as guaranteed by \cref{lem:v-streett-sm-spec-reward:lower-bound}%
. This hypothesis implies almost-sure Streett satisfaction within $I$, and is a uniform strengthening of the condition $\mathbb{E}_\pi[\mathrm{step}^{(A,B)} \mid S_0 = s] < \infty$, which was shown to be necessary and sufficient for the existence of a Streett supermartingale~\cite{kuraHierarchySupermartingalesoRegular2025}.



\begin{proof}[\Cref{thm:v-streett-sm-spec-reward}]
We need to show that $W(s) = C - V^\pi(s)$ is a Streett supermartingale (\Cref{def:streett-supermartingale}) for the Streett pair $(A, B)$.
Conditions~\eqref{def:streett-supermartingale:eq-invariant-init} and~\eqref{def:streett-supermartingale:eq-invariance} are identical to hypothesis~\ref{thm:v-streett-sm-spec-reward:hyp-invariant}, so we have to verify that $W(s) := C - V^\pi(s)$ is nonnegative and satisfies conditions~\eqref{def:streett-supermartingale:eq-strictly-decreasing} and~\eqref{def:streett-supermartingale:eq-non-increasing}.

\paragraph{Nonnegativity.}
The reward and state-dependent discount factor 
(hypothesis~\ref{thm:v-streett-sm-spec-reward:hyp-reward} \textbf{(\textit{reward})} and hypothesis~\ref{thm:v-streett-sm-spec-reward:hyp-discount}) 
induce (by \Cref{def:val-func}) the value function:
\begin{align}
V^{\pi}(s)
=
\mathbb{E}_\pi\!\left[
  \sum_{t=0}^\infty \gamma^{N_t^{(A\cup B)}}\,
  \bigl(\mathit{r_B}\,\mathbf{1}_{\{S_t\in B\}}
    - \mathit{r_{A\setminus B}}\,\mathbf{1}_{\{S_t\in A\setminus B\}}\bigr)
  \,\middle|\, S_0 = s
\right],
\end{align}
Since the reward is at most $\mathit{r_B}$ at every step,
it follows that $V^{\pi}(s) \le \mathit{r_B}\sum_{t=0}^\infty \gamma^t = C := \frac{\mathit{r_B}}{1-\gamma}$,
so $W(s) = C - V^{\pi}(s) \ge 0$ for all $s \in I$.

\paragraph{Condition~\eqref{def:streett-supermartingale:eq-strictly-decreasing}.}
Fix $s \in (A \setminus B) \cap I$. Since $s \in A \setminus B$, the immediate reward is $-r_{A\setminus B}$, so the Bellman equation (\Cref{def:val-func}) gives
\begin{align}
V^{\pi}(s)
=
-r_{A\setminus B}
+
\gamma\,\mathbb{E}_{\pi}[V^{\pi}(S_1)\mid S_0=s],
\end{align}
and therefore
\begin{align}\label{eqn:bellman-expansion-spec-reward}
\mathbb{E}_{\pi}[V^{\pi}(S_1)\mid S_0=s]
-
V^{\pi}(s)
=
r_{A\setminus B}
+
\mathbb{E}_{\pi}[(1-\gamma)\,V^{\pi}(S_1)\mid S_0=s].
\end{align}
\Cref{lem:v-streett-sm-spec-reward:lower-bound} lets us obtain the worst-case bound
\begin{align}\label{eqn:worst-case-bound-spec-reward}
\mathbb{E}_{\pi}[(1-\gamma)\,V^{\pi}(S_1)\mid S_0=s]
\;\ge\;
-(1-\gamma)\,r_{A\setminus B}\,\frac{1-\gamma^{\bar{H}}}{1-\gamma}
=
-r_{A\setminus B}\bigl(1-\gamma^{\bar{H}}\bigr).
\end{align}
Combining \eqref{eqn:bellman-expansion-spec-reward} and \eqref{eqn:worst-case-bound-spec-reward},
\begin{align}
\mathbb{E}_{\pi}[V^{\pi}(S_1)\mid S_0=s] - V^{\pi}(s)
\;\ge\;
r_{A\setminus B} - r_{A\setminus B}\bigl(1-\gamma^{\bar{H}}\bigr)
=
r_{A\setminus B}\,\gamma^{\bar{H}}
=:
\varepsilon.
\end{align}
Since $r_{A\setminus B} > 0$ and $\gamma^{\bar{H}} > 0$, we have $\varepsilon > 0$.
Substituting $W = C - V^\pi$ gives
\begin{align}
\mathbb{E}_\pi[W(S_1)\mid S_0=s]
\;\le\;
W(s) - \varepsilon,
\end{align}
which establishes condition~\eqref{def:streett-supermartingale:eq-strictly-decreasing}.

\paragraph{Condition~\eqref{def:streett-supermartingale:eq-non-increasing}.}
Fix $s \in I \setminus (A \cup B)$. For such $s$, the immediate reward is $0$ and $\Gamma(s) = 1$,
so the Bellman equation (\Cref{def:val-func}) gives
\begin{align}
V^{\pi}(s)
=
\mathbb{E}_{\pi}[\Gamma(s)\,V^{\pi}(S_1)\mid S_0=s]
=
\mathbb{E}_{\pi}[V^{\pi}(S_1)\mid S_0=s],
\end{align}
and therefore $\mathbb{E}_{\pi}[W(S_1)\mid S_0=s] = W(s)$,
which establishes condition~\eqref{def:streett-supermartingale:eq-non-increasing}.\qed
\end{proof}

Proof certificates for any $\omega$-regular (and LTL) specification $\varphi$ can be obtained by applying
Theorem~\ref{thm:v-streett-sm-spec-reward} to the product MDP
$\mathcal{M} \otimes \mathcal{A}_\varphi$, where the Streett pairs are directly inherited from the automaton.

\begin{corollary}
\label{cor:v-streett-sm-spec-reward-product}
Let $\mathcal{M} \otimes \mathcal{A}_\varphi$ be a product MDP with a DSA,
with state space $S \times Q$ and lifted Streett pairs $\mathit{Acc}^\otimes$.
Let $\pi$ be a stationary policy and let
$(S \times A, S \times B) \in \mathit{Acc}^\otimes$ be a lifted Streett pair.
Let $I \subseteq S \times Q$ be a supporting invariant. Assume that the hypotheses of
Theorem~\ref{thm:v-streett-sm-spec-reward} hold for the product MDP
$\mathcal{M} \otimes \mathcal{A}_\varphi$ in place of $\mathcal{M}$, the lifted pair
$(S \times A, S \times B)$ in place of $(A, B)$, and the reward
\[
r^\otimes\bigl((s,q), a, (s',q')\bigr)
= \mathit{r_B}\,\mathbf{1}_{\{(s,q)\in S \times B\}}
- \mathit{r_{A\setminus B}}\,\mathbf{1}_{\{(s,q)\in (S \times A)\setminus (S \times B)\}}.
\]
Then $W(s,q) := C - V^\pi(s,q)$, with $C := \tfrac{\mathit{r_B}}{1-\gamma}$,
is a Streett supermartingale for $(S \times A, S \times B)$ on
$\mathcal{M} \otimes \mathcal{A}_\varphi$.
\end{corollary}

\begin{example}
Consider again the controlled process presented in \Cref{ex:bscc-reward} and its supporting invariant.
Since $B = \emptyset$, the reward function from \Cref{thm:v-streett-sm-spec-reward} becomes
\begin{equation*}
r\bigl((x,q), a, (x',q')\bigr)
=
\begin{cases}
-r_{A\setminus B} & \text{if } q \in \{q_0, q_2\},\\
0 & \text{otherwise,}
\end{cases}
\end{equation*}
and the state-dependent discount factor is $\Gamma = \gamma$ when $q \in \{q_0, q_2\}$
and $\Gamma = 1$ when $q = q_1$.
Within the supporting invariant, the expected number of steps in $A \cup B = \{q_0, q_2\}$
before reaching $q_1$ is at most $7$, as explained in \Cref{ex:bscc-reward}.
This finite bound ($\bar{H} = 7$) establishes the positive recurrence hypothesis
of \Cref{thm:v-streett-sm-spec-reward}.
The values $r_B = 10$, $r_{A\setminus B} = 1$, and $\gamma = 0.99$ 
satisfy the required lower bound on $\frac{r_B}{r_{A\setminus B}}$.
The theorem yields the Streett supermartingale $W(s, q) := C - V^\pi(s, q)$,
with $C := \tfrac{r_B}{1-\gamma}$.
Outside $A \cup B$ ($q = q_1$), the reward is $0$ and $\Gamma = 1$,
so $V^\pi$ is unchanged in expectation and $W$ remains constant, satisfying the non-increasing condition \eqref{def:streett-supermartingale:eq-non-increasing} in \Cref{def:streett-supermartingale}.
Inside $A \setminus B$ ($q = q_0$), each step incurs the penalty $-r_{A\setminus B}$
and removes it from the discounted return, strictly increasing $V^\pi$.
The accumulated penalties are bounded by the finite hitting time,
ensuring $W$ strictly decreases with $\varepsilon = r_{A\setminus B}\,\gamma^{7}$, satisfying the strictly decreasing condition \eqref{def:streett-supermartingale:eq-strictly-decreasing} in \Cref{def:streett-supermartingale}. \qed 
\end{example}

Together, \Cref{thm:v-streett-supermartingale-bscc,thm:v-streett-sm-spec-reward} 
establish that, under either reward design and appropriate assumptions on 
the policy, the value function of a policy that almost surely satisfies 
an $\omega$-regular specification encodes a valid Streett supermartingale 
certificate. The tradeoffs between the two designs are discussed in 
\Cref{sec:conclusions}.

\section{Experimental Validation}
\label{sec:experimental-validation}



In this section, we validate \Cref{thm:v-streett-supermartingale-bscc,thm:v-streett-sm-spec-reward} 
on finite MDPs, where the transition kernel $P$ and reward function $r$ are 
fully known, allowing the exact computation of value functions. For each of 
a diverse set of LTL specifications, 
we synthesize proof certificates using both reward designs, verify them by 
checking the supermartingale proof rules, and confirm satisfaction probabilities 
independently via the PRISM model checker~\cite{kwiatkowskaPRISM40Verification2011}.
The code to reproduce our experiments is publicly available\footnote{https://github.com/danielcontro/v-ssm}. 
Experimental validation in infinite and continuous state spaces, where the value function must be approximated using parametric function classes such as neural networks, is left for future work.

\paragraph{MDP description.}
We consider a grid-based MDP over a finite grid with walls between certain pairs of adjacent cells, whose state space is $S = \bigl\{0,\ldots,19\bigr\}^2$.
At each state there are five available actions
$\act = \{\mathsf{up},\mathsf{down},\mathsf{left},\mathsf{right}, \mathsf{stay}\}$.
The transition kernel $P$ exhibits stochastic dynamics: executing a directional 
action $a \in \{\mathsf{up},\mathsf{down},\mathsf{left},\mathsf{right}\}$ moves 
to the intended successor with probability $2/3$, and to one of the two 
perpendicular successors with probability $1/6$ each. The $\mathsf{stay}$ action 
is deterministic and keeps the agent in its current cell. If the intended successor 
of a directional action is outside the grid or separated from $s$ by a wall, 
the agent stays in place instead.
Formally, for $s \in S$ and $a \in \act$,
\[
P(s' \mid s,a)=
\begin{cases}
1, & \text{if } a = \mathsf{stay} \text{ and } s = s', \\[2pt]
\frac{2}{3}, & \text{if } a \neq \mathsf{stay} \text{ and } s'=\mathit{next}(s,a), \\[2pt]
\frac{1}{6}, & \text{if } a \neq \mathsf{stay} \text{ and } s'=\mathit{next}(s,a_{\perp_1}) \text{ or } s'=\mathit{next}(s,a_{\perp_2}), \\[2pt]
0, & \text{otherwise}, 
\end{cases}
\]
where $\mathit{next}(s,a)$ denotes the neighbouring cell of $s$ in 
direction $a$, or $s$ itself if that cell is outside the grid or 
separated from $s$ by a wall; and $a_{\perp_1}, a_{\perp_2}$ denote 
the two actions perpendicular to $a$.
States are annotated with atomic propositions from 
$\mathrm{AP} = \{\mathsf{h}, \mathsf{a}, \mathsf{b}, \mathsf{c}, \mathsf{d}, \mathsf{e}\}$ 
via a labelling function $L : S \to 2^{\mathrm{AP}}$,
where $\mathsf{h}$ labels unsafe states (``holes'') that some specifications 
require to avoid, and $\mathsf{a}, \mathsf{b}, \mathsf{c}, \mathsf{d}, \mathsf{e}$ 
denote zones of interest appearing in the specifications.
A visual depiction of the environment and its labels is shown in \Cref{fig:maze-grid}.
\begin{figure}[ht]
    \vspace{-1em}
    \centering
    \begin{subfigure}[b]{0.45\textwidth}
        \centering
        \includegraphics[width=\textwidth]{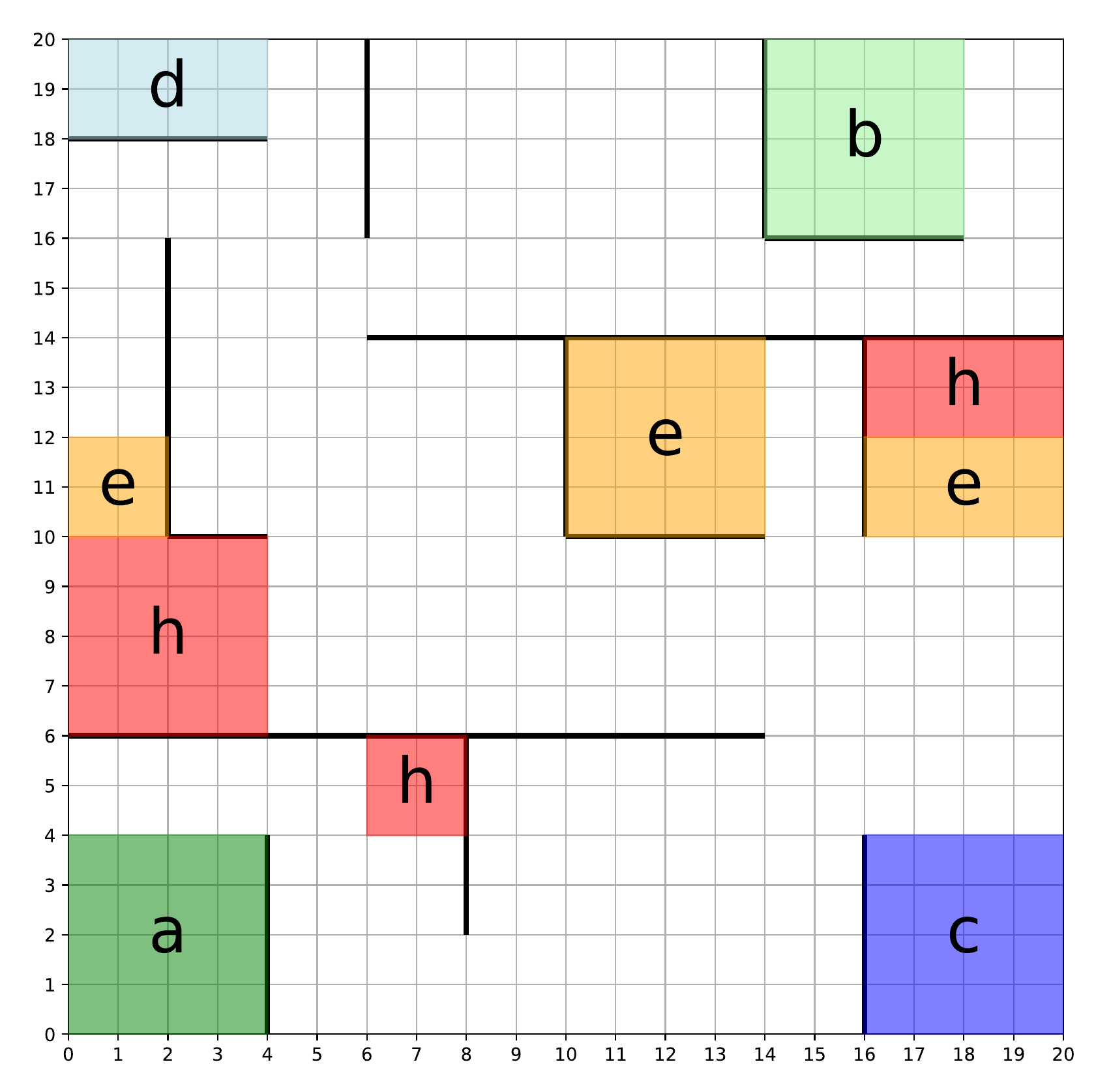}
        \caption{Maze layout.}
        \label{fig:maze-grid}
    \end{subfigure}
    \hfill
    \begin{subfigure}[b]{0.45\textwidth}
        \centering
        \includegraphics[width=\textwidth]{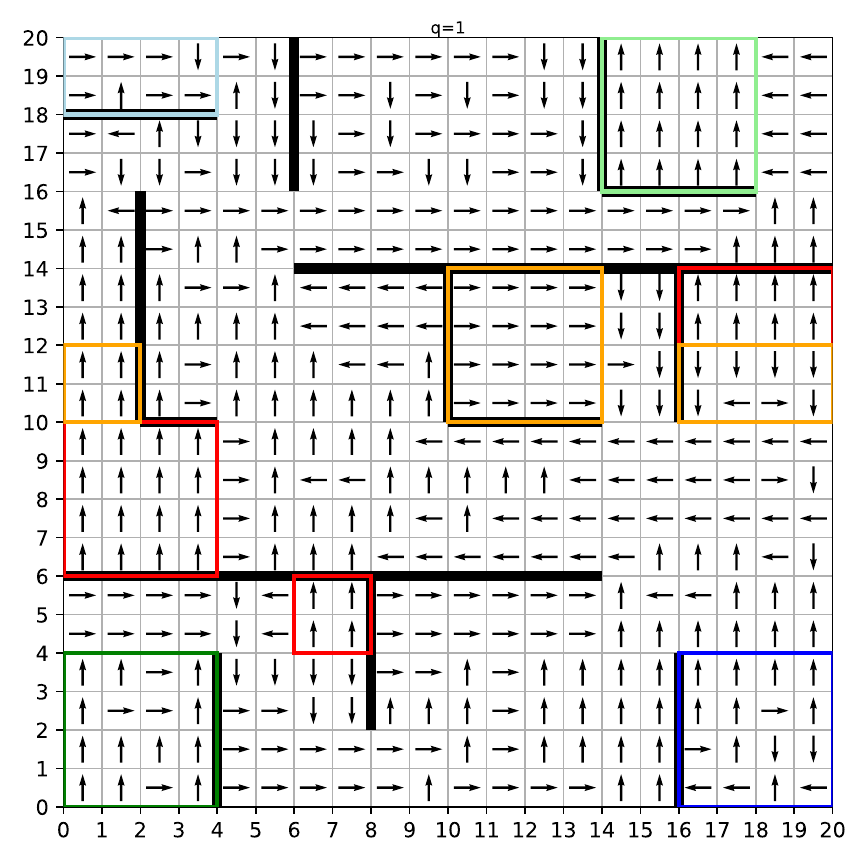}
        \caption{Policy for $\mathsf{F}\ \mathsf{b} \land \mathsf{G}\ \neg\mathsf{h}$.}
        \label{fig:maze-policy}
    \end{subfigure}
    \caption{The slippery maze environment used in our experiments. (\subref{fig:maze-grid}) shows the grid layout with labeled regions and impassable walls. (\subref{fig:maze-policy})  shows a policy that almost surely  reaches region $\mathsf{b}$ while permanently avoiding region $\mathsf{h}$, satisfying the specification $\mathsf{F}\ \mathsf{b} \land \mathsf{G}\ \neg\mathsf{h}$. Arrows indicate the agent's action at each state.}
    \label{fig:grid}
\end{figure}
\paragraph{Task specifications.}
We consider a diverse set of LTL specifications, listed in \Cref{tab:results}, 
covering every class in the Manna-Pnueli hierarchy \cite{mannaHierarchyTemporalProperties1990}. As an example for the rest of this section, consider $\mathsf{F}\ \mathsf{b} \land \mathsf{G}\ \neg\mathsf{h}$, which corresponds to the DSA in \Cref{fig:dsa-fb-gnoth}.
The DSA for each specification is provided in Appendix~\ref{appendix:automaa}.
\begin{figure}[h]
\centering
\begin{tikzpicture}[
    >=stealth,
    node distance=2cm,
    every state/.style={minimum size=0.5cm},
    every edge/.style={draw, ->},
    label/.style={font=\small},
]
\node[state, initial, initial text={}] (q1) {$q_1$};
\node[state, right of=q1] (q0) {$q_0$};
\node[state, right of=q0] (q2) {$q_2$};

\path[->] (q1) edge[loop above] node[above] {$\neg\mathsf{b} \land \neg\mathsf{h}$} (q1);
\path[->] (q0) edge[loop above] node[above] {$\neg\mathsf{h}$} (q0);
\path[->] (q2) edge[loop above] node[above] {$\mathrm{true}$} (q2);

\path[->] (q1) edge node[above] {$\mathsf{b} \land \neg\mathsf{h}$} (q0);

\path[->] (q1) edge[bend right=35] node[below] {$\mathsf{h}$} (q2);
\path[->] (q0) edge node[below] {$\mathsf{h}$} (q2);

\node[right=2em of q2, font=\normalsize] {$\mathrm{Acc} = \{(\{q_1, q_2\}, \{q_0\})\}$};
\end{tikzpicture}
\caption{Deterministic Streett Automaton for $\mathsf{F}\ \mathsf{b} \land \mathsf{G}\ \neg\mathsf{h}$}
\label{fig:dsa-fb-gnoth}
\vspace{-2em}
\end{figure}
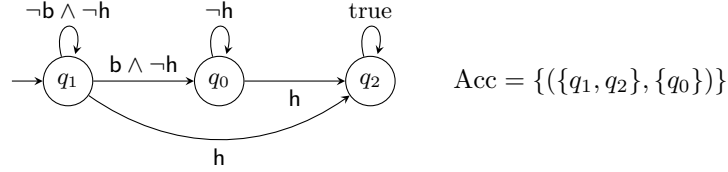%
\paragraph{Policies and supporting invariants.}
For each LTL task, we consider two deterministic policies: one that satisfies the task almost surely and one that does not. The former satisfies the assumptions of our theorems, while the latter serves as a negative witness, illustrating that the proof rules fail for non-satisfying policies, as expected. Figure~\ref{fig:maze-policy} shows the satisfying policy used for the task $\mathsf{F}\,\mathsf{b} \land \mathsf{G}\,\neg\mathsf{h}$. For each policy, we take as supporting invariant the set of states reachable from a fixed initial state in the product MDP, computed via reachability analysis.

\begin{wrapfigure}{r}{0.40\textwidth}
    \centering
    \vspace{-2em}
    \includegraphics[width=0.40\textwidth]{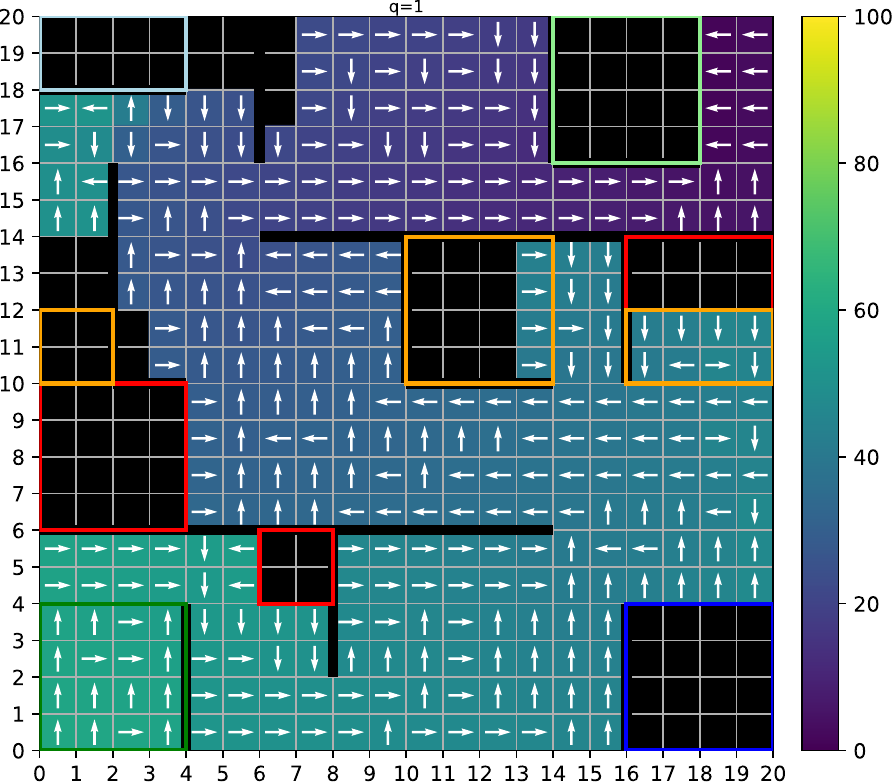}
    \caption{Synthesized proof certificate for $\mathsf{F}\,\mathsf{b} \land \mathsf{G}\,\neg\mathsf{h}$; black cells lie outside the invariant.}
    \label{fig:proof-certificate}
    \vspace{-2em}
\end{wrapfigure}
\paragraph{Proof certificate synthesis.}
For each specification, we construct a Streett automaton with $K$ Streett 
pairs. Each pair induces a reward function according to the design in 
\Cref{thm:v-streett-supermartingale-bscc}. Given a policy $\pi$, we evaluate it under each reward function independently by solving the linear system $(\mathbf{I} - \gamma P_\pi)V^\pi_i = R_i$,
where $R_i \in \mathbb{R}^{|S|}$ is the vector form of the reward function $r_i$ induced by the $i$-th Streett pair. This yields a set of value functions $\{V^\pi_1, \ldots, V^\pi_K\}$. For each Streett pair $i \in \{1, \ldots, K\}$, we then obtain a proof certificate by computing $W_i := \tfrac{1}{1-\gamma} - V^{\pi}_i$, as stated in \Cref{thm:v-streett-supermartingale-bscc}.
We repeat this procedure for the alternative reward design of 
\Cref{thm:v-streett-sm-spec-reward}, yielding a second set of proof certificates. 
We verify the synthesized proof certificates by checking the proof rules of \Cref{def:streett-supermartingale} at every state in the invariant, computing the post-expectation of the corresponding value function and checking whether the required conditions hold.
\Cref{fig:proof-certificate} visualizes a synthesized proof certificate for the task 
\(\mathsf{F}\ \mathsf{b} \land \mathsf{G}\ \neg\mathsf{h}\) over automaton state $q_1$ under the policy shown in \Cref{fig:maze-policy}. 
The values strictly decrease in expectation until the label $\mathsf{b}$ is observed, consistently with condition~\eqref{def:streett-supermartingale:eq-strictly-decreasing}, at which point the automaton transitions to $q_0$.
\paragraph{Validation against PRISM.}
To independently confirm the satisfaction status of each policy, we compute its satisfaction probability using the PRISM model checker~\cite{kwiatkowskaPRISM40Verification2011}, expecting a probability of $1$ for satisfying policies and strictly less than $1$ for non-satisfying ones. As reported in \Cref{tab:results}, the results are consistent with the theoretical guarantees: valid certificates are obtained for satisfying policies and violations are detected for non-satisfying ones.
\begin{table}[h]
\centering
\caption{Experimental results on the grid-world environment. For each specification, we evaluate a satisfying (Sat.) and a non-satisfying (Non-sat.) policy. A certificate is marked valid ($\checkmark$) if the proof rules are satisfied and invalid ($\times$) otherwise. The PRISM columns report the corresponding satisfaction probabilities.
The DSA for each specification is provided in Appendix~\ref{appendix:automaa}.
}
\label{tab:results}
\vspace{0.5em}
\setlength{\tabcolsep}{3pt}
\resizebox{0.90\textwidth}{!}{%
\begin{tabular}{lcccccc}
\toprule
\multirow{2}{*}{\textbf{Specification}} & \multirow{2}{*}{\makecell{\textbf{DSA} \\ \textbf{states}}} & \multirow{2}{*}{\makecell{\textbf{Streett} \\ \textbf{pairs}}} & \multicolumn{2}{c}{\textbf{Certificates valid}} & \multicolumn{2}{c}{\textbf{PRISM}} \\
\cmidrule(lr){4-5} \cmidrule(lr){6-7}
& & & \textbf{Sat.} & \textbf{Non-sat.} & \textbf{Sat.} & \textbf{Non-sat.} \\
\toprule
$\mathsf{F}\ \mathsf{b}$ 
& 2 & 1 & $\checkmark$ & $\times$ & $1.0$ &  0.1760 \\ 
\midrule
$\mathsf{F}\ \mathsf{b} \land \mathsf{G}\ \neg\mathsf{h}$ 
& 3 & 1 & $\checkmark$ & $\times$ & $1.0$ & 0.4532 \\ 
\midrule
$\mathsf{FG}\ \mathsf{e} \land \mathsf{G}\ \neg\mathsf{h}$ 
& 3 & 2 & $\checkmark$ & $\times$ & $1.0$ & 0.1051 \\ 
\midrule
$\mathsf{GF}\ \mathsf{b}$ 
& 2 & 1 & $\checkmark$ & $\times$ & $1.0$ & 0.2169 \\ 
\midrule
$\mathsf{GF}\ \mathsf{a} \land \mathsf{GF}\ \mathsf{b}  \land \mathsf{G}\ \neg \mathsf{h}$ 
& 6 & 1 & $\checkmark$ & $\times$ & $1.0$ & 0.8333 \\ 
\midrule
$\mathsf{GF}\ (\mathsf{a} \lor \mathsf{b}) \to \mathsf{FG}\ \neg \mathsf{e}$ 
& 7 & 2 & $\checkmark$ & $\times$ & $1.0$ &  0.0207 \\ 
\midrule
$\mathsf{F}\ (\mathsf{c} \land \mathsf{F}\ (\mathsf{d} \land \mathsf{GF}\ \mathsf{a}$
& 6 & 1 & $\checkmark$ & $\times$ & $1.0$ &  0.0771 \\ 
\quad\quad $ \land\ \mathsf{GF}\ \mathsf{b})) \land \mathsf{G}\ \neg\mathsf{h}$ 
& & & & & &
\\
\midrule
$\mathsf{GF}\ \mathsf{a} \to \mathsf{GF}\ (\mathsf{b} \lor  \mathsf{c})$ 
&3 & 1 & $\checkmark$ & $\times$ & $1.0$ &  0.0739 \\ 
\bottomrule
\end{tabular}
}
\end{table}

\section{Related Work}
\label{sec:related-work}

\paragraph{Verification of LTL properties in finite MDPs.} The verification of LTL properties in finite Markov Decision
Processes is an extensively studied topic in probabilistic model checking.
A range of efficient and scalable tools \cite{kwiatkowskaPRISM40Verification2011,DBLP:conf/cav/DehnertJK017,%
DBLP:conf/fm/HahnLSTZ14,DBLP:conf/qest/KatoenZHHJ09}
support this task.
%
These approaches typically proceed by translating the LTL formula into an equivalent $\omega$-automaton and taking its product with the MDP.
%
End-component analysis is then performed over the product to identify 
accepting components, from which satisfaction probabilities are computed 
via reachability analysis.
In contrast, our theoretical results are not limited to finite 
state and action spaces, encompassing countably infinite and continuous 
state spaces.

\paragraph{Infinite-state verification via supermartingales.}To enable verification and control of infinite state space stochastic systems, two main
lines of work have emerged in recent years: finite model abstractions via discretisation of 
 the state space \cite{%
 DBLP:conf/cav/AbateGS24,DBLP:conf/qest/DesharnaisLT08,DBLP:journals/siamads/SoudjaniA13,DBLP:conf/hybrid/TkachevA13,DBLP:conf/cav/ZhangSRHH10},
and approaches based on proof certificates   \cite{abateStochasticOmegaRegularVerification2024,%
DBLP:conf/cav/ChakarovS13,DBLP:conf/cav/ChatterjeeGMZ22,%
DBLP:conf/popl/ChatterjeeNZ17,DBLP:journals/automatica/Prajna06,DBLP:journals/toplas/TakisakaOUH21}.
In particular, the latter has seen significant development, where various 
proof rules based on (super)martingale theory have been proposed to establish almost-sure
satisfaction of specifications including reachability \cite{DBLP:conf/cav/ChakarovS13,DBLP:journals/pacmpl/Huang0CG19},
recurrence or persistence \cite{chakarovDeductiveProofsAlmost2016}, 
and more generally full $\omega$-regular properties \cite{abateStochasticOmegaRegularVerification2024}.
Beyond qualitative guarantees, more recently supermartingale-based proof
certificates have been presented for the verification of quantitative
specifications, providing lower bounds on satisfaction probabilities for
reachability \cite{DBLP:conf/tacas/ChatterjeeQSWWZ25,DBLP:conf/popl/ChatterjeeNZ17,DBLP:journals/pacmpl/MajumdarS25},
safety \cite{DBLP:conf/cav/ChatterjeeGMZ22,DBLP:conf/popl/ChatterjeeNZ17,DBLP:conf/cav/ZhiWLOZ24}, reach-avoidance \cite{DBLP:conf/tacas/ChatterjeeHLZ23},
persistence \cite{DBLP:journals/csysl/AjeleyeZ24a} and $\omega$-regular \cite{DBLP:conf/cav/AbateGR25,henzingerSupermartingaleCertificatesQuantitative2025} specifications%
; and, most recently, complete (and $\varepsilon$-complete) certificates for $\omega$-regular properties on countably infinite state spaces~\cite{abate2026completeomegaregularsupermartingalecertificates}.
Some algorithmic methods relying on such proof rules have been developed for
affine systems and certificates exploting Farkas' lemma \cite{DBLP:journals/pacmpl/AgrawalC018,DBLP:conf/cav/ChakarovS13,DBLP:conf/popl/ChatterjeeFNH16,DBLP:conf/popl/ChatterjeeNZ17}
and using Putinar's Positivstellensatz for polynomial programs and certificates \cite{DBLP:conf/cav/ChatterjeeFG16,DBLP:conf/cav/ChatterjeeGMZ22}.
Data-driven methods based on neural network templates have been developed 
for the synthesis of supermartingale certificates for 
termination~\cite{DBLP:conf/cav/AbateGR20}, quantitative 
reachability~\cite{DBLP:conf/concur/AbateEGPR23}, 
stability~\cite{DBLP:conf/aaai/LechnerZCH22}, and 
reach-avoidance~\cite{DBLP:conf/aaai/ZikelicLHC23,DBLP:conf/nips/ZikelicLVCH23,DBLP:conf/cav/BadingsKJJ25,DBLP:conf/ifaamas/DelgrangeAL0N025}; 
however, no data-driven method currently exists for the synthesis of 
supermartingale certificates for $\omega$-regular properties. 
Our theoretical results open a principled route to neural 
supermartingale certificate synthesis for $\omega$-regular properties via RL.


\paragraph{Reinforcement learning for temporal logic specifications.} Reinforcement learning under temporal logic specifications has
emerged as a framework for policy synthesis.
The standard approach translates the specification into an $\omega$-automaton and synchronises it on the fly with the environment during training. The policy is then learned over the product of the MDP and automaton states, with rewards shaped according to its acceptance condition.
Early work used DRA-based
encodings~\cite{sadighLearningBasedApproach2014},
while LDBAs were subsequently adopted as a more RL-friendly automaton
class~\cite{DBLP:journals/corr/abs-1801-08099,DBLP:conf/tacas/HahnPSSTW19,DBLP:conf/icra/Bozkurt0ZP20},
with further work refining reward schemes and establishing
optimality guarantees for model-free
algorithms~\cite{DBLP:conf/concur/KretinskyPR18,hahnFaithfulEffectiveReward2020,hasanbeigLCRLCertifiedPolicy2022,DBLP:conf/ijcai/ShaoK23}.
Extensions to continuous state spaces via deep RL have also been
explored~\cite{DBLP:conf/atal/HasanbeigAK19,DBLP:conf/formats/HasanbeigKA20}.
A complementary line of work enforces specifications at runtime via
shielding~\cite{DBLP:conf/aaai/AlshiekhBEKNT18}, which provides
guarantees for deployed policies but requires a model of the
environment.
Crucially, while reward-shaping approaches guarantee that the optimal
policy maximises the probability of satisfaction, they cannot certify
that a concrete trained policy---which may not be optimal---satisfies
the specification with any given probability, except in finite MDPs
where model checking can be applied post hoc.
In contrast, our work shows that, under an appropriate reward design,
the value function of a satisfying policy encodes a supermartingale
certificate, laying the theoretical groundwork for formal verification
of trained policies beyond finite state spaces.

\section{Towards Certified Reinforcement Learning}
\label{sec:conclusions}

We have established a novel theoretical connection between value functions and 
supermartingale certificates for the verification of LTL properties in 
stochastic systems. Specifically, we have shown that, under an appropriate 
reward design that depends on the Streett pair under analysis, the value function associated to a policy that almost surely satisfies 
an LTL specification encodes a valid Streett supermartingale certificate. These results hold for finite, countably infinite, and continuous 
state spaces and have been validated experimentally on finite MDPs.

\paragraph{On the assumptions in the theorems.}
The supermartingale conditions in \Cref{def:streett-supermartingale} can be verified directly on a given function $W$ — if they hold, the certificate is valid regardless of how $W$ was obtained or whether the assumptions of our theorems are satisfied.
The assumptions of \Cref{thm:v-streett-supermartingale-bscc,thm:v-streett-sm-spec-reward} (uniform hitting time bounds, almost-sure reachability, positive recurrence bound) play a different role, as they are sufficient conditions ensuring that the certificate derived from the value function under our reward design will satisfy these conditions. In practice, one may compute or approximate $V^\pi$ and directly verify the certificate conditions; if the check passes, the guarantee holds even if the theorem assumptions cannot be independently verified.
Both our theorems and the underlying proof rule also require a supporting invariant. The synthesis of such invariants is a standard challenge in supermartingale-based verification of $\omega$-regular properties~\cite{abateStochasticOmegaRegularVerification2024,kuraHierarchySupermartingalesoRegular2025} and lies outside the scope of this work.

\paragraph{Comparing the two reward designs.}
When the policy is obtained using existing RL methods for LTL tasks based on limit-deterministic B\"uchi automata (LDBAs)~\cite{DBLP:conf/formats/HasanbeigKA20,hasanbeigLCRLCertifiedPolicy2022}, where the learned policy solves the automaton's non-determinism, 
the acceptance condition reduces to B\"uchi pairs of the form $(S, T)$, and \Cref{thm:v-streett-supermartingale-bscc} can be applied via \Cref{cor:v-recurrence-supermartingale} without requiring 
knowledge of absorbing sets---only the target set $T$ from the acceptance condition is needed. Outside this setting, computing absorbing sets is generally unrealistic, and
\Cref{thm:v-streett-sm-spec-reward} provides an alternative that requires only knowledge of the specification. 
However, \Cref{thm:v-streett-sm-spec-reward} relies on a state-dependent discount factor, where transitions are discounted only within the rewarded region, and are undiscounted elsewhere. 
Although the Bellman operator under this discount structure is not a standard contraction, it is eventually contractive under the positive recurrence assumption, a setting known as ``eventual discounting''~\cite{stachurskiDynamicProgrammingStatedependent2021}, which has received attention in both dynamic programming and RL~\cite{suttonHordeScalableRealtime2011,yuConvergenceEmphaticTemporalDifference2017}.
The practical consequences for data-driven RL-based certificate synthesis are an open question.

\paragraph{Towards RL-based certificate synthesis.}
Our characterisation opens up a principled route to synthesising supermartingale certificates via reinforcement learning. Given a policy trained to satisfy the specification using
existing RL methods for LTL tasks~\cite{DBLP:conf/formats/HasanbeigKA20,hasanbeigLCRLCertifiedPolicy2022,sadighLearningBasedApproach2014}, one may estimate its value function using standard RL algorithms (e.g., temporal-difference  
learning) and verify the certificate conditions on the result, reducing certificate synthesis to a policy evaluation and a certificate check. 
Our experiments validate the theoretical connection in the exact-computation setting. A natural next step is to extend it to the data-driven setting, where the value function is estimated from sampled trajectories, introducing estimation error whose effect on certificate validity remains to be understood. A further challenge is scaling to continuous and high-dimensional state spaces, where the value function must be approximated using parametric function classes such as neural networks, compounding estimation error with approximation error.
Since our certificate is theoretically guaranteed to be valid under exact computation, verification failure in this setting would reflect errors in the value function approximation. How to diagnose and effectively correct them in practice to obtain a valid certificate remains an open question.

\begin{credits}
\subsubsection{\ackname} 
This research has been supported by EPSRC grant EP/Y028872/1, Mathematical Foundations of Intelligence: An ``Erlangen Programme'' for AI,
and funded by the Advanced Research + Invention Agency (ARIA) through the
project codes MSAI-PR01-P13 and MSAI-SE01-P012.

\end{credits}
%
%
%
\bibliographystyle{splncs04}
\bibliography{bibliography}
\appendix

\newpage
\section{Proofs}
\label{appendix:proofs}

\paragraph{Measurability and integrability.}
Under our MDP definition, the transition kernel \(P\), policy kernel \(\pi\),
and all sets such as \(U,A,B,I\) are measurable. The controlled process
\((S_t)_{t\ge 0}\) is therefore defined on the canonical path space
\((S^\omega,\mathcal{B}(S)^{\otimes\omega})\) under~\(\Pr_\pi\), and for any
measurable \(U\subseteq S\) the first hitting time
\(\tau_U := \inf\{t\ge0 : S_t\in U\}\) 
and the random variable \( S_{\tau_U(S_n)} \)
are measurable. 
Since rewards are bounded and \(\gamma<1\), all discounted
returns are integrable, and every expectation appearing below is well defined.
In the state-dependent case, \(\Gamma(s) = 1\) only at states where \(r = 0\), so undiscounted transitions merely propagate the value without contributing to the sum; the geometric decay \(\gamma < 1\) at all rewarded states ensures integrability is preserved.

\lemmaDecomposeV*
\begin{proof}[\Cref{lem:decompose-v}]
By \Cref{def:val-func}, $V^\pi(s)
=
\mathbb{E}_\pi\!\left[
   \sum_{t=0}^\infty 
   \gamma^t\,\mathbf{1}_{\{S_t \in U\}}
   \,\middle|\, S_0 = s
\right].$
On the event $\{\tau_U = \infty\}$, every indicator $\mathbf{1}_{\{S_t \in U\}}$ 
vanishes and the sum is zero, consistent with the right-hand side since 
$\mathbf{1}_{\{\tau_U < \infty\}}$ vanishes.
On the event $\{\tau_U < \infty\}$, all indicators $\mathbf{1}_{\{S_t \in U\}}$ 
vanish for $t < \tau_U$, and from time $\tau_U$ onward the process starts in 
$S_{\tau_U} \in U$ and accumulates discounted reward equal to $V^\pi(S_{\tau_U})$. 
Hence, the following pathwise identity holds almost surely:
\begin{align}
\sum_{t=0}^\infty \gamma^t \mathbf{1}_{\{S_t \in U\}}
=
\gamma^{\tau_U} 
V^\pi(S_{\tau_U})
\mathbf{1}_{\{\tau_U<\infty\}}.
\end{align}
Taking expectations yields the stated identity. Note that when $s \in U$, 
we have $\tau_U = 0$ and both sides reduce to $V^\pi(s)$, 
so the identity holds for all $s \in S$.
\qed
\end{proof}

\lemmaDeltaFormula*
\begin{proof}[\Cref{lem:delta-formula}]
Fix $s \in S \setminus U$. We invoke \Cref{lem:decompose-v} first to the random
initial state $S_1$ and then to the deterministic initial state $s$,
and subtract the two equalities, taking conditional expectation given $S_0 = s$
throughout.

On the event $\{\tau_U(S_1) < \infty\}$, both hitting times are finite and pathwise
\begin{align}
\label{eq:delta-proof-1}
\Delta := V^\pi(S_1) - V^\pi(s)
=
\gamma^{\tau_U(S_1)}V^\pi(S_{\tau_U(S_1)})
-
\gamma^{\tau_U}V^\pi(S_{\tau_U}).
\end{align}
By assumption $s \notin U$, hence the first visit to $U$ starting from $s$ occurs exactly one
step after the first visit from $S_1$ along the same trajectory. Hence, pathwise,
$\tau_U = 1 + \tau_U(S_1)$ and $S_{\tau_U} = S_{\tau_U(S_1)}$.
Substituting in \eqref{eq:delta-proof-1} gives
\begin{align*}
\Delta
=
\gamma^{\tau_U(S_1)}V^\pi(S_{\tau_U(S_1)})
-
\gamma^{\tau_U(S_1)+1}V^\pi(S_{\tau_U(S_1)})
=
(1-\gamma)\,\gamma^{\tau_U(S_1)}V^\pi(S_{\tau_U(S_1)}).
\end{align*}
On the event $\{\tau_U(S_1) = \infty\}$, $U$ is never reached from $S_1$ and
therefore never from $s$ along the same trajectory. Since, by assumption,
$r(s, a, s') = \mathbf{1}_{\{s \in U\}}$, no reward is ever collected, so both 
$V^\pi(S_1)$ and $V^\pi(s)$ are $0$ and $\Delta = 0$.

Combining the two events yields the pathwise identity
\begin{align}
\Delta
=
(1-\gamma)\,\gamma^{\tau_U(S_1)}V^\pi(S_{\tau_U(S_1)})
\mathbf{1}_{\{\tau_U(S_1)<\infty\}}.
\end{align}
Taking conditional expectation given $S_0 = s$ produces the stated formula.
\qed
\end{proof}

\lemmaVStrictlyIncreasing*
\begin{proof}[\Cref{lem:v-strictly-increasing}]
Fix $s \in I \setminus U$. We invoke \Cref{lem:delta-formula} to obtain
\begin{align}
\Delta(s)
= \mathbb{E}_\pi\!\Big[(1-\gamma)\,\gamma^{\tau_U(S_1)}
   \,V^\pi(S_{\tau_U(S_1)})\,
   \mathbf{1}_{\{\tau_U(S_1)<\infty\}}\;\Big|\;S_0=s\Big].
\end{align}
Since $I$ is absorbing and $s \in I \setminus U$, we have $S_1 \in I$ 
almost surely. By hypothesis, for every $x \in I \setminus U$ the hitting time 
$\tau_U$ is finite with probability one, and for $x \in U$ we have $\tau_U = 0$. 
Hence $\Pr_\pi(\tau_U(S_1) < \infty \mid S_0 = s) = 1$, so the indicator 
$\mathbf{1}_{\{\tau_U(S_1) < \infty\}}$ equals $1$ almost surely.
Since, by assumption, $r(s, a, s') = \mathbf{1}_{\{s \in U\}}$, any $q \in U$ 
collects immediate reward $1$ at every step, so $V^\pi(q) \ge 1$, which gives 
$V^\pi(S_{\tau_U(S_1)}) \ge 1$ on the event $\{\tau_U(S_1) < \infty\}$.
Therefore
\begin{align}
\Delta(s)
\ge \mathbb{E}_\pi\!\Big[(1-\gamma)\,\gamma^{\tau_U(S_1)}\;\Big|\;S_0=s\Big].
\end{align}
With $0 < \gamma < 1$, the integrand $(1-\gamma)\gamma^{\tau_U(S_1)}$ is strictly 
positive almost surely, so the conditional expectation is strictly positive.
Therefore $\Delta(s) > 0$, as required.
\qed
\end{proof}

\lemmaVNonIncreasing*
\begin{proof}[\Cref{lem:v-non-increasing}]
Fix $s \in U$. By assumption, $\Pr_\pi(S_1 \in U \mid S_0 = s) = 1$ for all 
$s \in U$. Applying this inductively, $\Pr_\pi(S_t \in U \mid S_0 = s) = 1$ 
for all $t \geq 0$. Since, by assumption, $r(s, a, s') = 1$ for all $s \in U$, 
$a, s'$, the reward at each step is $1$ almost surely. Then, applying \Cref{def:val-func}, we obtain
\begin{align*}
V^\pi(s) = \sum_{t=0}^\infty \gamma^t = \frac{1}{1-\gamma}.
\end{align*}
Since $S_1 \in U$ almost surely, the same argument applies and yields
$V^\pi(S_1) = \frac{1}{1-\gamma}$ almost surely. Therefore
\begin{align*}
\mathbb{E}_\pi\!\big[V^\pi(S_1)\mid S_0=s\big] - V^\pi(s)
= \frac{1}{1-\gamma} - \frac{1}{1-\gamma}
= 0.
\end{align*}
\qed
\end{proof}

\lemmaVStrictlyIncreasingUniformBound*
\begin{proof}[\Cref{lem:v-strictly-increasing-uniform-bound}]
Fix $s \in F \setminus U$.
We invoke \Cref{lem:delta-formula} to obtain
\begin{align}
\Delta(s)
= 
\mathbb{E}_\pi\!\Big[
   (1-\gamma)\,\gamma^{\tau_U(S_1)}\,
   V^\pi(S_{\tau_U(S_1)})\,
   \mathbf{1}_{\{\tau_U(S_1)<\infty\}}
   \;\Big|\; S_0=s
\Big].
\end{align}
Since, by assumption, $r(s, a, s') = \mathbf{1}_{\{s \in U\}}$, any $q \in U$
collects immediate reward $1$ at every step, so $V^\pi(q) \ge 1$.
Since rewards are nonnegative we also have $V^\pi \ge 0$.
Therefore, on the event $\{\tau_U(S_1) < \infty\}$, we have 
$V^\pi(S_{\tau_U(S_1)}) \ge 1$, and thus
\begin{align}
\label{eq:proof-uniform-bound-1}
\Delta(s)
\ge 
\mathbb{E}_\pi\!\Big[
   (1-\gamma)\,\gamma^{\tau_U(S_1)}\,
   \mathbf{1}_{\{\tau_U(S_1)<\infty\}}
   \;\Big|\;S_0=s
\Big].
\end{align}
Since $s \notin U$, we have $\tau_U = 1 + \tau_U(S_1)$ whenever $\tau_U < \infty$.
By hypothesis, $\mathbb{E}_\pi[\tau_U \mid S_0 = s] \le \bar H$,
which implies $\mathbb{E}_\pi\!\big[\tau_U(S_1) \mid S_0 = s\big] \le \bar H - 1$.
Because the function $h \mapsto \gamma^h$ is convex for $0 < \gamma < 1$,
Jensen's inequality yields
\begin{align}
\label{eq:proof-uniform-bound-2}
\mathbb{E}_\pi\!\big[\gamma^{\tau_U(S_1)} \mid S_0=s\big]
\;\ge\;
\gamma^{\,\mathbb{E}_\pi[\tau_U(S_1)\mid S_0=s]}
\;\ge\;
\gamma^{\bar H - 1}.
\end{align}
Combining \eqref{eq:proof-uniform-bound-1} and \eqref{eq:proof-uniform-bound-2} we obtain
\begin{align}
\Delta(s)
\ge (1-\gamma)\,\gamma^{\bar H - 1}
=: \varepsilon
> 0,
\end{align}
which proves the claim.
\qed
\end{proof}

\lemmaSpecRewardLowerBound*
We need the following definition for the next proof. For a measurable set $U \subseteq S$ and $k \in \mathbb{N}_{\ge 1}$, let
$\tau_{U,k} : S^\omega \to \mathbb{N} \cup \{\infty\}$ denote the time of the $k$-th
visit to $U$, defined for a run $\rho = (s_0, s_1, s_2, \ldots) \in S^\omega$ by
\[
\tau_{U,k}(\rho)
:=
\inf
\left\{
t \in \mathbb{N}
:
\left|
\left\{
i \mid 0 \le i \le t \land s_i \in U
\right\}
\right|
\ge k
\right\}.
\]
\begin{proof}[\Cref{lem:v-streett-sm-spec-reward:lower-bound}]
From the reward definition and discount factor we have
\[
V^{\pi}(s)
=
\mathbb{E}_\pi\!\left[
  \sum_{t=0}^\infty \gamma^{N_t^{(A \cup B)}}\, 
  (\mathit{r_B}\,\mathbf{1}_{\{s\in B\}}
    - \mathit{r_{A\setminus B}}\,\mathbf{1}_{\{s\in A\setminus B\}})
  \,\middle|\, S_0 = s
\right].
\]
By hypothesis
the discount 
accumulates only on steps where reward is nonzero, i.e.\ steps in 
$A \cup B$, so the effective discount at time $t$ is 
$\gamma^{N_t^{(A \cup B)}}$.
Define
\(
E_1 = \left\{\sum_{t=0}^\infty \mathbf{1}_A(S_t) = \infty\right\} \cap \left\{\sum_{t=0}^\infty \mathbf{1}_B(S_t) = \infty\right\}
\)
and
\(
E_2 = \left\{\sum_{t=0}^\infty \mathbf{1}_A(S_t) < \infty\right\}.
\)
Since $I$ is absorbing with $\mu(I)=1$ and
there exists $\bar{H} < \infty$ such that
\(
\forall s \in I:
\mathbb{E}_{\pi}\left[\mathrm{step}^{(A,B)} \mid S_0 = s\right] \leq \bar{H}
\), from the hypothesis
we have that $\Pr_\pi(E_1\cup E_2)=1$.

\paragraph{Case \(E_1\).} Conditional on $E_1$, the process visits \(B\) infinitely often.
Decomposing the return over successive
visits to \(B\), we obtain
\begin{align*}
V^{\pi}(s)
&=
\sum_{i\ge 1}
\mathbb{E}_\pi\!\left[
  -
  r_{A\setminus B}
  \left(
  \sum_{t=\tau_{B,i-1}+1}^{\tau_{B, i}-1}
  \gamma^{N_t^{(A \cup B)}}
  \mathbf{1}_{\{S_t\in A\setminus B\}}
  \right) 
  +
  r_B\,\gamma^{N_{\tau_{B, i}}^{(A \cup B)}}
\right].
\end{align*}
For each $i$, the summand captures the negative return accumulated in 
$A \setminus B$ between the $(i-1)$-th and $i$-th visits to $B$, plus 
the positive return upon reaching $B$.
%
Re-indexing the inner sum by factoring out 
$\gamma^{N_{\tau_{B, i-1}+1}^{(A \cup B)}}$ and using 
$N_{\tau_{B, i}-1}^{(A \cup B)} - N_{\tau_{B, i-1}+1}^{(A \cup B)} 
= \mathrm{step}^{(A,B)}_{\tau_{B, i-1}+1}$
yields\footnote{More precisely, 
$\mathrm{step}^{(A,B)}_{\tau_{B, i-1}+1} 
:= \mathrm{step}^{(A,B)} \circ \theta_{\tau_{B, i-1}+1}$, 
where $\theta_\tau : S^\omega \to S^\omega$ is the shift operator 
applied $\tau$ times, i.e.\ 
$\theta_\tau(s_0, s_1, \ldots) = (s_\tau, s_{\tau+1}, \ldots)$.
}
\begin{align*}
V^{\pi}(s)
&=
\sum_{i\ge 1}
\mathbb{E}_\pi\!\left[
  -
  r_{A\setminus B}
  \left(
  \sum_{k=N_{\tau_{B, i-1}+1}^{(A \cup B)}}^{N_{\tau_{B, i}-1}^{(A \cup B)}}
  \gamma^{k}
  \right)
  +
  r_B\,\gamma^{N_{\tau_{B, i}}^{(A \cup B)}}
\right]
\\
&=
\sum_{i\ge 1}
\mathbb{E}_\pi\!\left[
  -
  r_{A\setminus B}\,
  \gamma^{N_{\tau_{B, i-1}+1}^{(A \cup B)}}
  \left(
  \sum_{k=0}^{\mathrm{step}^{(A,B)}_{{\tau_{B, i-1}+1}}}
  \gamma^{k}
  \right)
  +
  r_B\,\gamma^{N_{\tau_{B, i}}^{(A \cup B)}}
\right].
\end{align*}

Factoring out $\gamma^{N_{\tau_{B, i-1}+1}^{(A \cup B)}}$ and using 
$N_{\tau_{B, i}}^{(A \cup B)} = N_{\tau_{B, i-1}+1}^{(A \cup B)} + \mathrm{step}^{(A,B)}_{\tau_{B, i-1}+1} + 1$,
\begin{align*}
&=
\sum_{i\ge 1}
\mathbb{E}_\pi\!\left[
  \gamma^{N_{\tau_{B, i-1}+1}^{(A \cup B)}}
  \left(
    r_B\,\gamma^{\mathrm{step}^{(A,B)}_{{\tau_{B, i-1}+1}}+1}
    -
    r_{A\setminus B}
    \sum_{k=0}^{\mathrm{step}^{(A,B)}_{{\tau_{B, i-1}+1}}}
    \gamma^{k}
  \right)
\right].
\end{align*}

Using the strong Markov property at \(\tau_{B, i-1}+1\),
\begin{align*}
&=
\sum_{i\ge 1}
\mathbb{E}_\pi\!\left[
  \gamma^{N_{\tau_{B, i-1}+1}^{(A \cup B)}}
\right]
\,
\mathbb{E}_\pi\!\left[
  r_B\,\gamma^{\mathrm{step}^{(A,B)}_{\tau_{B, i-1}+1}+1}
  -
  r_{A\setminus B}
  \sum_{k=0}^{\mathrm{step}^{(A,B)}_{\tau_{B, i-1}+1}}
  \gamma^{k}
\right].
\end{align*}

Evaluating the geometric sum we obtain
\begin{align*}
&=
\sum_{i\ge 1}
\mathbb{E}_\pi\!\left[
  \gamma^{N_{\tau_{B, i-1}+1}^{(A \cup B)}}
\right]
\,
\mathbb{E}_\pi\!\left[
  r_B\,\gamma^{\mathrm{step}^{(A,B)}_{\tau_{B, i-1}+1}+1}
  -
  r_{A\setminus B}
  \frac{1-\gamma^{\mathrm{step}^{(A,B)}_{\tau_{B, i-1}+1}+1}}
       {1-\gamma}
\right].
\end{align*}

By Jensen (applied to the convex map \(m\mapsto\gamma^{m}\)), we may lower-bound the 
inner expectation in terms of the expectation of step:
\begin{align*}
&\geq
\sum_{i\ge 1}
\mathbb{E}_\pi\!\left[
  \gamma^{N_{\tau_{B, i-1}+1}^{(A \cup B)}}
\right]
\,
\left(
r_B\,\gamma^{\mathbb{E}_\pi[\mathrm{step}^{(A,B)}_{\tau_{B, i-1}+1}]+1}
-
r_{A\setminus B}\,
\frac{1-\gamma^{\mathbb{E}_\pi[\mathrm{step}^{(A,B)}_{\tau_{B, i-1}+1}]+1}}
     {1-\gamma}   
\right).
\end{align*}

Note that
\(
\mathbb{E}_{\pi}\left[\mathrm{step}^{(A,B)} \mid S_0 = s\right] \leq \bar{H}
\) by hypothesis.
Applying this yields the lower bound
\begin{align*}
V^{\pi}(s)
&\ge
\sum_{i\ge 1}
\mathbb{E}_\pi\!\left[
  \gamma^{N_{\tau_{B, i-1}+1}^{(A \cup B)}}
\right]
\,
\left(
  r_B\,\gamma^{\bar{H}+1}
  -
  r_{A\setminus B}
  \frac{1-\gamma^{\bar{H}+1}}
       {1-\gamma}
\right).
\end{align*}
By the hypothesis on the reward ratio,
\[
\frac{r_B}{r_{A\setminus B}}
\;\ge\;
\frac{1-\gamma^{\bar{H}+1}}{(1-\gamma)\gamma^{\bar{H}+1}}
\;=\;
\frac{1}{1-\gamma}\!\left(\frac{1}{\gamma^{\bar{H}+1}}-1\right),
\]
which rearranges to
\[
r_B\,\gamma^{\bar{H}+1}
\;-\;
r_{A\setminus B}\,\frac{1-\gamma^{\bar{H}+1}}{1-\gamma}
\;\ge\; 0.
\]
Since $\gamma^{N_{\tau_{B, i-1}+1}^{(A \cup B)}} > 0$ for every $i$, each term in the series is non-negative, and we conclude that on the event $E_1$,
\[
V^\pi(s) \;\ge\; 0.
\]

\paragraph{Case \(E_2\).}

Conditional on $E_2$, the process visits \(A\) only finitely many times.
Let
\[
L := \sup\{t \ge 0 \mid S_t \in A\}
\]
be the (a.s. finite) last visit time to \(A\).
After time \(L\), the process never visits \(A\setminus B\), therefore
all subsequent rewards are nonnegative. Therefore,
\[
V^\pi(s)
\ge
\mathbb{E}_\pi\!\left[
  \sum_{t=0}^{L}
  \gamma^{N_t^{(A \cup B)}}
  \bigl(
    r_B \mathbf{1}_{\{S_t\in B\}}
    -
    r_{A\setminus B}\mathbf{1}_{\{S_t\in A\setminus B\}}
  \bigr)
  \,\middle|\, S_0=s
\right].
\]

Let
\(
\ell := \sup\{t \le L \mid S_t \in B\}
\)
be the last visit to \(B\) before time \(L\)
(with the convention \(\ell=-1\) if no such visit exists).
Split the finite sum at time \(\ell\):
\[
\sum_{t=0}^{\ell}
\gamma^{N_t^{(A \cup B)}}
\bigl(
  r_B \mathbf{1}_{\{S_t\in B\}}
  -
  r_{A\setminus B}\mathbf{1}_{\{S_t\in A\setminus B\}}
\bigr)
+
\sum_{t=\ell+1}^{L}
\gamma^{N_t^{(A \cup B)}}
\bigl(
  r_B \mathbf{1}_{\{S_t\in B\}}
  -
  r_{A\setminus B}\mathbf{1}_{\{S_t\in A\setminus B\}}
\bigr).
\]

The first sum covers the trajectory up to the last visit to $B$, 
which can be decomposed into segments between consecutive visits to $B$.
Applying the same argument as in Case~$E_1$ to each such segment, 
we obtain the lower bound
\(
r_B\,\gamma^{\bar H + 1}
-
r_{A\setminus B}\frac{1-\gamma^{\bar H+1}}{1-\gamma}
\;\ge\; 0
\)
for each segment.
Hence,
\[
\mathbb{E}_\pi 
\left[
\sum_{t=0}^{\ell}
\gamma^{N_t^{(A \cup B)}}
  \bigl(
    r_B \mathbf{1}_{\{S_t\in B\}}
    -
    r_{A\setminus B}\mathbf{1}_{\{S_t\in A\setminus B\}}
  \bigr)
\right]
\;\ge\; 0.
\]

It remains to bound the tail from \(\ell+1\) to \(L\).
Between $\ell+1$ and $L$ there are no visits to $B$ by definition of $\ell$,
so the $r_B$ term vanishes:
\begin{align*}
& \mathbb{E}_\pi
\left[
\sum_{t=\ell+1}^{L}
\gamma^{N_t^{(A \cup B)}}
\bigl(
  r_B \mathbf{1}_{\{S_t\in B\}}
  -
  r_{A\setminus B}\mathbf{1}_{\{S_t\in A\setminus B\}}
\bigr)
\right]
\\
& =
  -
  r_{A\setminus B}\ 
\mathbb{E}_\pi
\left[
\sum_{t=\ell+1}^{L}
\gamma^{N_t^{(A \cup B)}}
\mathbf{1}_{\{S_t\in A\setminus B\}}
\right].
\end{align*}
Between $\ell+1$ and $L$ all visits to $A\cup B$ are visits to $A\setminus B$,
so factoring out $\gamma^{N_\ell^{(A,B)}}$ and re-indexing by the number of
$A\setminus B$-steps after $\ell$:
\begin{align*}
& =
  -
  r_{A\setminus B}\ 
\mathbb{E}_\pi
\left[
\gamma^{N_\ell^{(A \cup B)}}
\sum_{k=0}^{\mathrm{step}_{\ell}^{(A,B)}-1}
\gamma^k
\right].
\end{align*}
Evaluating the geometric sum,
\begin{align*}
& =
  -
  r_{A\setminus B}\ 
\mathbb{E}_\pi
\left[
\gamma^{N_\ell^{(A \cup B)}}
\frac{1-\gamma^{\mathrm{step}_{\ell}^{(A,B)}}}{1-\gamma}
\right].
\end{align*}
By the strong Markov property at $\ell$, $\gamma^{N_\ell^{(A \cup B)}}$ and
$\mathrm{step}^{(A,B)}_\ell$ are independent, so the expectation factors:
\begin{align*}
& =
  -
  r_{A\setminus B}\ 
\mathbb{E}_\pi
\left[
\gamma^{N_\ell^{(A \cup B)}}
\right]
\mathbb{E}_\pi
\left[
\frac{1-\gamma^{\mathrm{step}_{\ell}^{(A,B)}}}{1-\gamma}
\right].
\end{align*}
Note that the expression being bounded carries an overall negative sign, so each subsequent step replaces a factor by an upper bound to obtain a lower bound on the whole.
Since $m\mapsto\gamma^m$ is convex, Jensen's inequality gives
$\mathbb{E}[\gamma^{\mathrm{step}}] \geq \gamma^{\mathbb{E}[\mathrm{step}]}$,
hence $\frac{1-\mathbb{E}[\gamma^{\mathrm{step}}]}{1-\gamma}
\leq \frac{1-\gamma^{\mathbb{E}[\mathrm{step}]}}{1-\gamma}$, and thus obtaining
\begin{align*}
& \;\ge\;
  -
  r_{A\setminus B}\ 
\mathbb{E}_\pi
\left[
\gamma^{N_\ell^{(A \cup B)}}
\right]
\frac{1-\gamma^{\mathbb{E}_\pi[\mathrm{step}_{\ell}^{(A,B)}]}}{1-\gamma}.
\end{align*}
Applying $\mathbb{E}_\pi[\mathrm{step}^{(A,B)}_\ell] \leq \bar{H}$
from hypothesis
and monotonicity of
$m\mapsto\gamma^m$:
\begin{align*}
& \;\ge\;
  -
  r_{A\setminus B}\ 
\mathbb{E}_\pi
\left[
\gamma^{N_\ell^{(A \cup B)}}
\right]
\frac{1-\gamma^{\bar{H}}}{1-\gamma}.
\end{align*}
Finally, using $\mathbb{E}_\pi[\gamma^{N_\ell^{(A \cup B)}}] \leq 1$ since $\gamma\in(0,1)$:
\begin{align*}
& \;\ge\;
  -
  r_{A\setminus B}\ 
\frac{1-\gamma^{\bar{H}}}{1-\gamma}.
\end{align*}

\paragraph{Combining the cases.}
Since the contribution of Case~$E_1$ to $V^{\pi}(s)$ is non-negative and the contribution of Case~$E_2$ is at least $-r_{A\setminus B}\,\frac{1-\gamma^{\bar H}}{1-\gamma}\cdot\Pr_\pi(E_2) \;\geq\; -r_{A\setminus B}\,\frac{1-\gamma^{\bar H}}{1-\gamma}$, we conclude
\[
V^{\pi}(s)
\ge
-\, r_{A\setminus B}\,\frac{1-\gamma^{\bar H}}{1-\gamma}.
\]
\qed
\end{proof}










\newpage
\section{Deterministic Streett Automata for Experimental Validation}
\label{appendix:automaa}
We list the Deterministic Streett Automata for the specifications in \Cref{tab:results}.
\subsection{\textsf{F b}}

\begin{figure}
\begin{center}
\begin{tikzpicture}[shorten >=1pt, node distance=3cm, on grid, auto, initial text=]

\node[state, initial] (q0) {$q_0$};
\node[state, right=of q0] (q1) {$q_1$};

\path[->]
(q0) edge[loop above] node {$\neg\textsf{b}$} ()
(q0) edge node[above] {$\textsf{b}$} (q1)
(q1) edge[loop above] node {$true$} ();

\node at (current bounding box.south) [below=5mm] {
$Acc=\{(\{q_0\},\{q_1\})\}$
};
\end{tikzpicture}
\end{center}
  \caption{DSA for $\textsf{F b}$}
  \label{fig:Fb}
\end{figure}

\subsection{$\textsf{F b} \land \textsf{G }\neg\textsf{h}$}

\begin{figure}[!h]
\begin{center}
\begin{tikzpicture}[shorten >=1pt, node distance=3cm, on grid, auto, initial text=]

\node[state, initial] (q0) {$q_0$};
\node[state, right=of q0, yshift=-15mm] (q1) {$q_1$};
\node[state, right=of q1, yshift=15mm] (q2) {$q_2$};

\path[->]
(q0) edge[loop above] node {$\neg\textsf{b} \land \neg\textsf{h}$} ()
(q2) edge[loop above] node {$true$} ()
(q0) edge[bend left] node[above] {$\textsf{h}$} (q2)
(q1) edge node[below] {$\textsf{h}$} (q2)
(q0) edge node[below, xshift=-4mm] {$\textsf{b} \land \neg\textsf{h}$} (q1)
(q1) edge[loop above] node {$\neg\textsf{h}$} ();

\node at (current bounding box.south) [below=5mm] {
$Acc=\{(\{q_0, q_2\},\{q_1\})\}$
};
\end{tikzpicture}
\end{center}
\caption{DSA for $\textsf{F b}\land\textsf{G }\neg\textsf{h}$}
\label{fig:Fb_and_G_neg_h}
\end{figure}
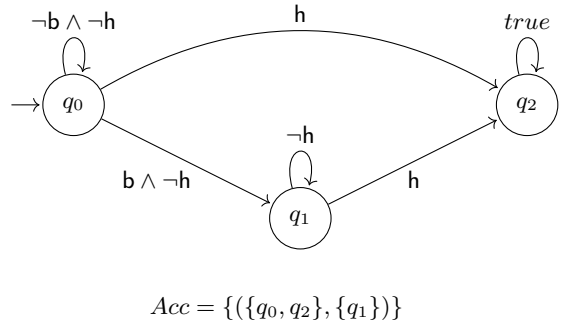

\newpage

\subsection{$\textsf{FG e} \land \textsf{G }\neg\textsf{h}$}

\begin{figure}[!h]
\begin{center}
\begin{tikzpicture}[shorten >=1pt, node distance=3cm, on grid, auto, initial text=]

\node[state, initial] (q0) {$q_0$};
\node[state, right=of q0, yshift=-15mm] (q1) {$q_1$};
\node[state, right=of q1, yshift=15mm] (q2) {$q_2$};

\path[->]
(q0) edge[loop above] node {$\neg\textsf{e} \land \neg\textsf{h}$} ()
(q2) edge[loop above] node {$true$} ()
(q0) edge[bend left] node[above] {$\textsf{h}$} (q2)
(q1) edge node[below] {$\textsf{h}$} (q2)
(q0) edge[bend left] node[above,yshift=1mm] {$\textsf{e} \land \neg\textsf{h}$} (q1)
(q1) edge[bend left] node[below,yshift=-1mm] {$\neg\textsf{e} \land \neg\textsf{h}$} (q0)
(q1) edge[loop above] node {$\textsf{e} \land \neg\textsf{h}$} ();

\node at (current bounding box.south) [below=5mm] {
$Acc=\{(\{q_0, q_2\},\emptyset), (\{q_0, q_2\}, \{q_1\})\}$
};

\end{tikzpicture}
\end{center}
\caption{DSA for $\textsf{FG e}\land\textsf{G }\neg\textsf{h}$}
\label{fig:FG_e_and_G_neg_h}
\end{figure}

\subsection{$\textsf{GF b}$}

\begin{figure}[!h]
\begin{center}
\begin{tikzpicture}[shorten >=1pt, node distance=3cm, on grid, auto, initial text=]

\node[state, initial] (q0) {$q_0$};
\node[state, right=of q0] (q1) {$q_1$};

\path[->]
(q0) edge[loop above] node {$\textsf{b}$} ()
(q0) edge[bend left] node[above] {$\neg\textsf{b}$} (q1)
(q1) edge[bend left] node[below] {$\textsf{b}$} (q0)
(q1) edge[loop above] node {$\neg\textsf{b}$} ();

\node at (current bounding box.south) [below=5mm] {
$Acc=\{(\{q_1\},\{q_0\})\}$
};
\end{tikzpicture}
\end{center}
\caption{DSA for $\textsf{GF b}$}
\label{fig:GF_a_and_GF_b_and_G_neg_h}
\end{figure}

\newpage

\subsection{$\textsf{GF a} \land \textsf{GF b} \land \textsf{G }\neg\textsf{h}$}

\begin{figure}[!h]
\begin{center}
\begin{tikzpicture}[shorten >=1pt, node distance=4cm, on grid, auto, initial text=,every node/.style={font=\scriptsize}]

\node[state, initial] (q0) {$q_0$};
\node[state, above right=of q0] (q1) {$q_1$};
\node[state, right=of q1] (q2) {$q_2$};
\node[state, below right=of q2] (q3) {$q_3$};
\node[state, below right=of q0] (q5) {$q_5$};
\node[state, right=of q5] (q4) {$q_4$};

\path[->]
  (q0) edge[loop above] node[left] {$\textsf{a} \land \textsf{b} \land \neg\textsf{h}$} ()
  (q2) edge[loop right] node[yshift=-2mm, xshift=-3mm] {$\neg\textsf{a} \land \neg\textsf{h}$} ()
  (q3) edge[loop right] node[right] {$true$} ()
  (q4) edge[loop below] node {$\textsf{a} \land \neg\textsf{b} \land \neg\textsf{h}$} ()
  (q5) edge[loop below] node {$\neg\textsf{a} \land \neg\textsf{b} \land \neg\textsf{h}$} ();

\path[->]
  (q0) edge node[xshift=30mm, yshift=-4mm] {$h$} (q3)
  (q0) edge[bend left] node {$\textsf{a} \land \neg\textsf{b} \land \neg\textsf{h}$} (q1)

  (q1) edge[bend left] node[left, xshift=3mm, yshift=4mm] {$\textsf{a} \land \textsf{b} \land \neg\textsf{h}$} (q0)
  (q1) edge[bend left=20] node {$\neg\textsf{a} \land \textsf{b} \land \neg\textsf{h}$} (q2)
  (q1) edge[bend left=80] node {$\textsf{h}$} (q3)
  (q1) edge node[xshift=-4mm, yshift=5mm] {$\textsf{a} \land \neg\textsf{b} \land \neg\textsf{h}$} (q4)
  (q1) edge node[yshift=-6mm] {$\neg\textsf{a} \land \neg\textsf{b} \land \neg\textsf{h}$} (q5)

  (q2) edge node[below, xshift=18mm, yshift=6mm] {$\textsf{a} \land \textsf{b} \land \neg\textsf{h}$} (q0)
  (q2) edge[bend left=20] node[above] {$\textsf{a} \land \neg\textsf{b} \land \neg\textsf{h}$} (q1)
  (q2) edge node[above] {$\textsf{h}$} (q3)

  (q4) edge node[below, xshift=-17mm, yshift=6mm] {$\textsf{b} \land \neg\textsf{h}$} (q0)
  (q4) edge node[above] {$\textsf{h}$} (q3)
  (q4) edge node[below] {$\neg\textsf{a} \land \neg\textsf{b} \land \neg\textsf{h}$} (q5)

  (q5) edge[bend left] node {$\textsf{a} \land \textsf{b} \land \neg\textsf{h}$} (q0)
  (q5) edge[bend right=50] node[xshift=25mm, yshift=14mm] {$\neg\textsf{a} \land \textsf{b} \land \neg\textsf{h}$} (q2)
  (q5) edge[bend right] node[below] {$\textsf{a} \land \neg\textsf{b} \land \neg\textsf{h}$} (q4)
  (q5) edge node[below, xshift=12mm, yshift=5mm] {$\textsf{h}$} (q3);

\node at (current bounding box.south) [below=5mm] {
$Acc=\{(\{q_2, q_3, q_4, q_5\},\{q_0, q_1\})\}$
};
\end{tikzpicture}
\end{center}
\caption{DSA for $\textsf{GF a}\land\textsf{GF b}\land\textsf{G }\neg\textsf{h}$}
\label{fig:GF_a_GF_b_G_neg_h}
\end{figure}

\newpage

\subsection{$\textsf{GF (a} \lor \textsf{b)} \rightarrow \textsf{FG } \neg \textsf{e}$}

\begin{figure}[!h]
\begin{center}
\begin{tikzpicture}[shorten >=1pt, node distance=4cm, on grid, auto, initial text=,every node/.style={font=\scriptsize}]

\node[state, initial] (q0) {$q_0$};
\node[state, above right=of q0] (q1) {$q_1$};
\node[state, below right=of q0] (q2) {$q_2$};
\node[state, right=of q1] (q6) {$q_6$};
\node[state, right=of q2] (q3) {$q_3$};
\node[state, above right=of q3] (q5) {$q_5$};
\node[state, right=35mm of q0] (q4) {$q_4$};

\path[->]
  (q0) edge[loop above] node[left] {$(\textsf{a}\land\textsf{e}) \lor (\textsf{b}\land\textsf{e})$} ()
  (q4) edge[loop below] node {$\neg \textsf{e}$} ()
  (q5) edge[loop right] node {$\neg \textsf{a} \land \neg \textsf{b}$} ();

\path[->]
  (q1) edge[bend left=20] node[xshift=7mm, yshift=12mm] {$(\textsf{a}\land\textsf{e}) \lor (\textsf{b}\land\textsf{e})$} (q0)
  (q2) edge[bend left=20] node {$(\textsf{a}\land\textsf{e}) \lor (\textsf{b}\land\textsf{e})$} (q0)
  (q3) edge node[xshift=15mm, yshift=-5mm] {$(\textsf{a}\land\textsf{e}) \lor (\textsf{b}\land\textsf{e})$} (q0)
  (q4) edge node[below, xshift=2mm] {$(\textsf{a}\land\textsf{e}) \lor (\textsf{b}\land\textsf{e})$} (q0)
  (q5) edge[bend right=15] node[above, xshift=-5mm] {$(\textsf{a}\land\textsf{e}) \lor (\textsf{b}\land\textsf{e})$} (q0)
  (q6) edge node[xshift=5mm, yshift=3mm] {$(\textsf{a}\land\textsf{e}) \lor (\textsf{b}\land\textsf{e})$} (q0)

  (q0) edge[bend left=20] node {$\neg\textsf{a}\land\neg\textsf{b}\land\textsf{e}$} (q1)

  (q0) edge[bend left=20] node[left] {$\neg\textsf{e}$} (q2)

  (q2) edge node[below] {$\neg\textsf{a}\land\neg\textsf{b}\land\textsf{e}$} (q3)
  (q4) edge node[xshift=-5mm, yshift=2mm] {$\neg\textsf{a}\land\neg\textsf{b}\land\textsf{e}$} (q3)
  (q6) edge[bend left=10] node[yshift=-8mm] {$\neg\textsf{a}\land\neg\textsf{b}\land\textsf{e}$} (q3)

  (q1) edge[bend left=20] node[xshift=10mm, yshift=-6mm] {$\neg \textsf{a} \land \neg \textsf{b}$} (q5)
  (q3) edge[bend right] node[right] {$\neg \textsf{a} \land \neg \textsf{b}$} (q5)

  (q1) edge node {$(\textsf{a}\land\neg\textsf{e}) \lor (\textsf{b}\land\neg\textsf{e})$} (q6)
  (q3) edge[bend left=10] node {$(\textsf{a}\land\neg\textsf{e}) \lor (\textsf{b}\land\neg\textsf{e})$} (q6)
  (q5) edge[bend right] node[right] {$(\textsf{a}\land\neg\textsf{e}) \lor (\textsf{b}\land\neg\textsf{e})$} (q6)
  ;

\node at (current bounding box.south) [below=5mm] {
$Acc=\{(\{q_0,q_3,q_6\}, \emptyset), (\{q_0,q_1,q_2,q_3,q_6\}, \{q_4,q_5\})\}$
};
\end{tikzpicture}
\end{center}
\caption{DSA for $\textsf{GF }(\textsf{a}\lor\textsf{b}) \rightarrow \textsf{FG }\neg\textsf{e}$}
\label{fig:GF_a_or_b_then_FG_neg_e}
\end{figure}
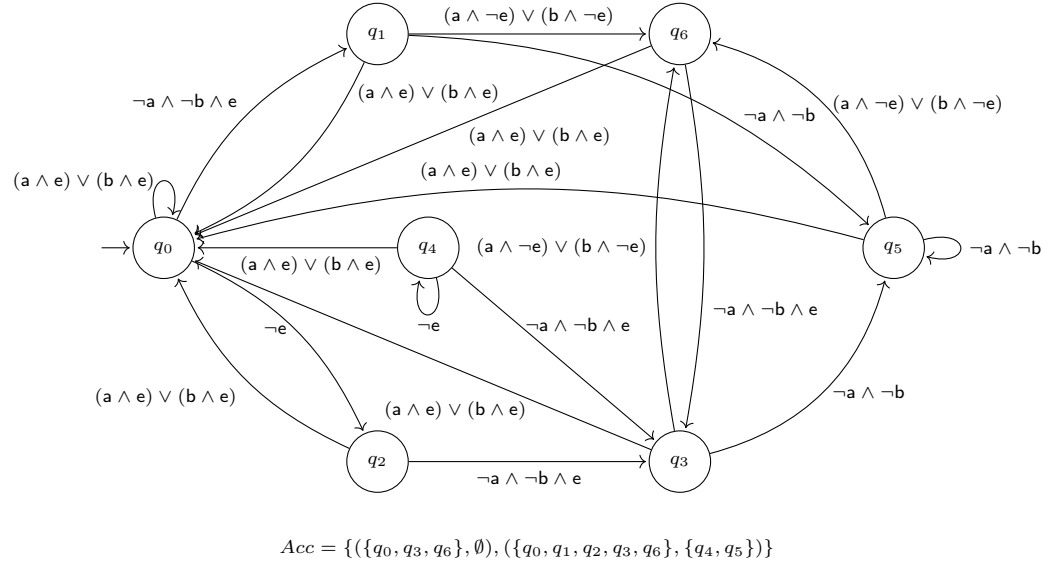

\newpage

\subsection{$\textsf{F }(\textsf{c} \land\textsf{F }(\textsf{d} \land \textsf{GF a} \land \textsf{GF b})) \land \textsf{G }\neg\textsf{h}$}

\begin{figure}[!h]
\begin{center}
\begin{tikzpicture}[shorten >=1pt, node distance=4cm, on grid, auto, initial text=,every node/.style={font=\scriptsize}]

\node[state, initial] (q0) {$q_0$};
\node[state, above right=of q0] (q1) {$q_1$};
\node[state, below right=of q0] (q2) {$q_2$};
\node[state, right=of q1] (q5) {$q_5$};
\node[state, right=of q2] (q4) {$q_4$};
\node[state, below right=of q5] (q3) {$q_3$};

\path[->]
  (q0) edge[loop above] node {$\neg\textsf{c}\land\neg\textsf{h}$} ()
  (q1) edge[loop above] node {$\neg\textsf{d}\land\neg\textsf{h}$} ()
  (q2) edge[loop below] node {$\textsf{a}\land\textsf{b}\land\neg\textsf{h}$} ()
  (q3) edge[loop above] node {$true$} ()
  (q4) edge[loop below] node {$\neg\textsf{b}\land\neg\textsf{h}$} ()
  (q5) edge[loop above] node {$\neg\textsf{a}\land\neg\textsf{h}$} ()
  ;

\path[->]
(q0) edge node {$\textsf{c}\land\neg\textsf{d}\land\neg\textsf{h}$} (q1)

(q0) edge node[left] {$\textsf{c}\land\textsf{d}\land\neg\textsf{h}$} (q2)
(q1) edge node[yshift=12mm] {$\textsf{d}\land\neg\textsf{h}$} (q2)
(q4) edge[bend left=20] node {$\textsf{a}\land\textsf{b}\land\neg\textsf{h}$} (q2)
(q5) edge[bend left=25] node[left] {$\textsf{a}\land\neg\textsf{h}$} (q2)

(q2) edge[bend left=20] node[below] {$\neg\textsf{b}\land\neg\textsf{h}$} (q4)

(q2) edge[bend left=25] node[right] {$\neg\textsf{a}\land\textsf{b}\land\neg\textsf{h}$} (q5)
(q4) edge node[right, yshift=6mm] {$\neg\textsf{a}\land\textsf{b}\land\neg\textsf{h}$} (q5)

(q0) edge node[xshift=-30mm] {$\textsf{h}$} (q3)
(q1) edge[bend left=20] node[xshift=-20mm, yshift=5mm] {$\textsf{h}$} (q3)
(q2) edge node {$\textsf{h}$} (q3)
(q4) edge node {$\textsf{h}$} (q3)
(q5) edge node {$\textsf{h}$} (q3)
;

\node at (current bounding box.south) [below=5mm] {
$Acc=\{(\{q_0,q_1,q_3,q_4,q_5\},\{q_2\})\}$
};
\end{tikzpicture}
\end{center}
\caption{DSA for $\textsf{F }(\textsf{c}\land\textsf{F }(\textsf{d} \land \textsf{GF a} \land \textsf{GF b})) \land \textsf{G }\neg\textsf{h}$}
\label{fig:F_c_and_F_d_and_GF_a_and_GF_b_and_G_neg_h}
\end{figure}
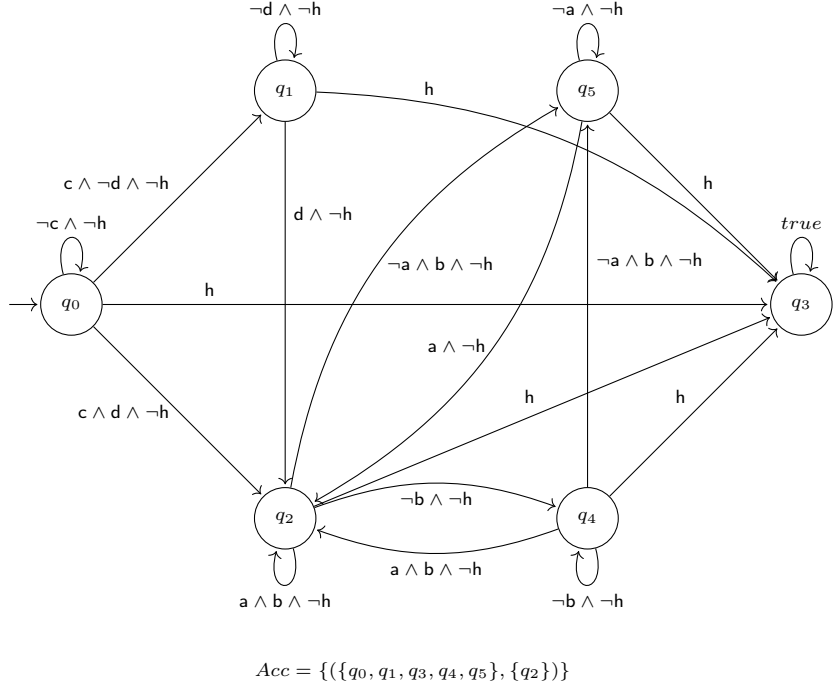

\newpage

\subsection{$\textsf{GF a} \rightarrow \textsf{GF (b} \lor \textsf{c)}$}

\begin{figure}[!h]
\begin{center}
\begin{tikzpicture}[shorten >=1pt, node distance=3cm, on grid, auto, initial text=,every node/.style={font=\small}]

\node[state, initial] (q0) {$q_0$};
\node[state, right=of q0, yshift=20mm] (q1) {$q_1$};
\node[state, right=of q1, yshift=-20mm] (q2) {$q_2$};

\path[->]
(q0) edge[loop above] node {$\textsf{b} \lor \textsf{c}$} ()
(q1) edge[loop above] node {$\textsf{a} \land \neg\textsf{b} \land \neg\textsf{c}$} ()
(q2) edge[loop right] node[right] {$\neg\textsf{a} \land \neg\textsf{b} \land \neg\textsf{c}$} ();

\path[->]
(q0) edge[bend left] node[above, yshift=3mm, xshift=-2mm] {$\textsf{a} \land \neg\textsf{b} \land \neg\textsf{c}$} (q1)
(q1) edge[bend left] node[above, yshift=3mm, xshift=2mm] {$\neg\textsf{a} \land \neg\textsf{b} \land \neg\textsf{c}$} (q2)
(q0) edge node[below] {$\neg\textsf{a} \land \neg\textsf{b} \land \neg\textsf{c}$} (q2)
(q2) edge[bend left] node[below] {$\textsf{b} \lor \textsf{c}$} (q0)
(q1) edge[bend left] node[above, xshift=-1mm, yshift=2mm] {$\textsf{b} \lor \textsf{c}$} (q0)
(q2) edge[bend left] node[above, xshift=5mm, yshift=2mm] {$\textsf{a} \land \neg \textsf{b} \land \neg\textsf{c}$} (q1);

\node at (current bounding box.south) [below=5mm] {
$Acc=\{(\{q_1\},\{q_0\})\}$
};
\end{tikzpicture}
\end{center}
\caption{DSA for $\textsf{GF a}\rightarrow\textsf{GF }(\textsf{b}\lor\textsf{c})$}
\label{fig:GF_a_then_GF_b_or_c}
\end{figure}

\end{document}